\documentclass[twoside,11pt]{article}

\usepackage{jmlr2e}
\usepackage{def}
\usepackage{caption}
\usepackage{color}
\usepackage{algorithm}
\usepackage[noend]{algorithmic}

\jmlrheading{1}{2021}{1-48}{4/00}{10/00}{meila00a}{Fan Yao, Chuanhao Li, Denis Nekipelov, Hongning Wang and Haifeng Xu}

\ShortHeadings{Learning the Optimal Recommendation from Explorative Users}{Learning the Optimal Recommendation from Explorative Users}
\firstpageno{1}

\begin{document}

\newcommand{\figleft}{{\em (Left)}}
\newcommand{\figcenter}{{\em (Center)}}
\newcommand{\figright}{{\em (Right)}}
\newcommand{\figtop}{{\em (Top)}}
\newcommand{\figbottom}{{\em (Bottom)}}
\newcommand{\captiona}{{\em (a)}}
\newcommand{\captionb}{{\em (b)}}
\newcommand{\captionc}{{\em (c)}}
\newcommand{\captiond}{{\em (d)}}

\newcommand{\newterm}[1]{{\bf #1}}

\def\figref#1{figure~\ref{#1}}
\def\Figref#1{Figure~\ref{#1}}
\def\twofigref#1#2{figures \ref{#1} and \ref{#2}}
\def\quadfigref#1#2#3#4{figures \ref{#1}, \ref{#2}, \ref{#3} and \ref{#4}}
\def\secref#1{section~\ref{#1}}
\def\Secref#1{Section~\ref{#1}}
\def\twosecrefs#1#2{sections \ref{#1} and \ref{#2}}
\def\secrefs#1#2#3{sections \ref{#1}, \ref{#2} and \ref{#3}}
\def\plaineqref#1{\ref{#1}}
\def\chapref#1{chapter~\ref{#1}}
\def\Chapref#1{Chapter~\ref{#1}}
\def\rangechapref#1#2{chapters\ref{#1}--\ref{#2}}
\def\algref#1{algorithm~\ref{#1}}
\def\Algref#1{Algorithm~\ref{#1}}
\def\twoalgref#1#2{algorithms \ref{#1} and \ref{#2}}
\def\Twoalgref#1#2{Algorithms \ref{#1} and \ref{#2}}
\def\partref#1{part~\ref{#1}}
\def\Partref#1{Part~\ref{#1}}
\def\twopartref#1#2{parts \ref{#1} and \ref{#2}}

\def\ceil#1{\lceil #1 \rceil}
\def\floor#1{\lfloor #1 \rfloor}
\def\1{\bm{1}}
\newcommand{\train}{\mathcal{D}}
\newcommand{\valid}{\mathcal{D_{\mathrm{valid}}}}
\newcommand{\test}{\mathcal{D_{\mathrm{test}}}}

\def\eps{{\epsilon}}

\def\reta{{\textnormal{$\eta$}}}
\def\ra{{\textnormal{a}}}
\def\rb{{\textnormal{b}}}
\def\rc{{\textnormal{c}}}
\def\rd{{\textnormal{d}}}
\def\re{{\textnormal{e}}}
\def\rf{{\textnormal{f}}}
\def\rg{{\textnormal{g}}}
\def\rh{{\textnormal{h}}}
\def\ri{{\textnormal{i}}}
\def\rj{{\textnormal{j}}}
\def\rk{{\textnormal{k}}}
\def\rl{{\textnormal{l}}}
\def\rn{{\textnormal{n}}}
\def\ro{{\textnormal{o}}}
\def\rp{{\textnormal{p}}}
\def\rq{{\textnormal{q}}}
\def\rr{{\textnormal{r}}}
\def\rs{{\textnormal{s}}}
\def\rt{{\textnormal{t}}}
\def\ru{{\textnormal{u}}}
\def\rv{{\textnormal{v}}}
\def\rw{{\textnormal{w}}}
\def\rx{{\textnormal{x}}}
\def\ry{{\textnormal{y}}}
\def\rz{{\textnormal{z}}}

\def\rvepsilon{{\mathbf{\epsilon}}}
\def\rvtheta{{\mathbf{\theta}}}
\def\rva{{\mathbf{a}}}
\def\rvb{{\mathbf{b}}}
\def\rvc{{\mathbf{c}}}
\def\rvd{{\mathbf{d}}}
\def\rve{{\mathbf{e}}}
\def\rvf{{\mathbf{f}}}
\def\rvg{{\mathbf{g}}}
\def\rvh{{\mathbf{h}}}
\def\rvu{{\mathbf{i}}}
\def\rvj{{\mathbf{j}}}
\def\rvk{{\mathbf{k}}}
\def\rvl{{\mathbf{l}}}
\def\rvm{{\mathbf{m}}}
\def\rvn{{\mathbf{n}}}
\def\rvo{{\mathbf{o}}}
\def\rvp{{\mathbf{p}}}
\def\rvq{{\mathbf{q}}}
\def\rvr{{\mathbf{r}}}
\def\rvs{{\mathbf{s}}}
\def\rvt{{\mathbf{t}}}
\def\rvu{{\mathbf{u}}}
\def\rvv{{\mathbf{v}}}
\def\rvw{{\mathbf{w}}}
\def\rvx{{\mathbf{x}}}
\def\rvy{{\mathbf{y}}}
\def\rvz{{\mathbf{z}}}

\def\erva{{\textnormal{a}}}
\def\ervb{{\textnormal{b}}}
\def\ervc{{\textnormal{c}}}
\def\ervd{{\textnormal{d}}}
\def\erve{{\textnormal{e}}}
\def\ervf{{\textnormal{f}}}
\def\ervg{{\textnormal{g}}}
\def\ervh{{\textnormal{h}}}
\def\ervi{{\textnormal{i}}}
\def\ervj{{\textnormal{j}}}
\def\ervk{{\textnormal{k}}}
\def\ervl{{\textnormal{l}}}
\def\ervm{{\textnormal{m}}}
\def\ervn{{\textnormal{n}}}
\def\ervo{{\textnormal{o}}}
\def\ervp{{\textnormal{p}}}
\def\ervq{{\textnormal{q}}}
\def\ervr{{\textnormal{r}}}
\def\ervs{{\textnormal{s}}}
\def\ervt{{\textnormal{t}}}
\def\ervu{{\textnormal{u}}}
\def\ervv{{\textnormal{v}}}
\def\ervw{{\textnormal{w}}}
\def\ervx{{\textnormal{x}}}
\def\ervy{{\textnormal{y}}}
\def\ervz{{\textnormal{z}}}

\def\rmA{{\mathbf{A}}}
\def\rmB{{\mathbf{B}}}
\def\rmC{{\mathbf{C}}}
\def\rmD{{\mathbf{D}}}
\def\rmE{{\mathbf{E}}}
\def\rmF{{\mathbf{F}}}
\def\rmG{{\mathbf{G}}}
\def\rmH{{\mathbf{H}}}
\def\rmI{{\mathbf{I}}}
\def\rmJ{{\mathbf{J}}}
\def\rmK{{\mathbf{K}}}
\def\rmL{{\mathbf{L}}}
\def\rmM{{\mathbf{M}}}
\def\rmN{{\mathbf{N}}}
\def\rmO{{\mathbf{O}}}
\def\rmP{{\mathbf{P}}}
\def\rmQ{{\mathbf{Q}}}
\def\rmR{{\mathbf{R}}}
\def\rmS{{\mathbf{S}}}
\def\rmT{{\mathbf{T}}}
\def\rmU{{\mathbf{U}}}
\def\rmV{{\mathbf{V}}}
\def\rmW{{\mathbf{W}}}
\def\rmX{{\mathbf{X}}}
\def\rmY{{\mathbf{Y}}}
\def\rmZ{{\mathbf{Z}}}

\def\ermA{{\textnormal{A}}}
\def\ermB{{\textnormal{B}}}
\def\ermC{{\textnormal{C}}}
\def\ermD{{\textnormal{D}}}
\def\ermE{{\textnormal{E}}}
\def\ermF{{\textnormal{F}}}
\def\ermG{{\textnormal{G}}}
\def\ermH{{\textnormal{H}}}
\def\ermI{{\textnormal{I}}}
\def\ermJ{{\textnormal{J}}}
\def\ermK{{\textnormal{K}}}
\def\ermL{{\textnormal{L}}}
\def\ermM{{\textnormal{M}}}
\def\ermN{{\textnormal{N}}}
\def\ermO{{\textnormal{O}}}
\def\ermP{{\textnormal{P}}}
\def\ermQ{{\textnormal{Q}}}
\def\ermR{{\textnormal{R}}}
\def\ermS{{\textnormal{S}}}
\def\ermT{{\textnormal{T}}}
\def\ermU{{\textnormal{U}}}
\def\ermV{{\textnormal{V}}}
\def\ermW{{\textnormal{W}}}
\def\ermX{{\textnormal{X}}}
\def\ermY{{\textnormal{Y}}}
\def\ermZ{{\textnormal{Z}}}

\def\vzero{{\bm{0}}}
\def\vone{{\bm{1}}}
\def\vmu{{\bm{\mu}}}
\def\vtheta{{\bm{\theta}}}
\def\va{{\bm{a}}}
\def\vb{{\bm{b}}}
\def\vc{{\bm{c}}}
\def\vd{{\bm{d}}}
\def\ve{{\bm{e}}}
\def\vf{{\bm{f}}}
\def\vg{{\bm{g}}}
\def\vh{{\bm{h}}}
\def\vi{{\bm{i}}}
\def\vj{{\bm{j}}}
\def\vk{{\bm{k}}}
\def\vl{{\bm{l}}}
\def\vm{{\bm{m}}}
\def\vn{{\bm{n}}}
\def\vo{{\bm{o}}}
\def\vp{{\bm{p}}}
\def\vq{{\bm{q}}}
\def\vr{{\bm{r}}}
\def\vs{{\bm{s}}}
\def\vt{{\bm{t}}}
\def\vu{{\bm{u}}}
\def\vv{{\bm{v}}}
\def\vw{{\bm{w}}}
\def\vx{{\bm{x}}}
\def\vy{{\bm{y}}}
\def\vz{{\bm{z}}}

\def\evalpha{{\alpha}}
\def\evbeta{{\beta}}
\def\evepsilon{{\epsilon}}
\def\evlambda{{\lambda}}
\def\evomega{{\omega}}
\def\evmu{{\mu}}
\def\evpsi{{\psi}}
\def\evsigma{{\sigma}}
\def\evtheta{{\theta}}
\def\eva{{a}}
\def\evb{{b}}
\def\evc{{c}}
\def\evd{{d}}
\def\eve{{e}}
\def\evf{{f}}
\def\evg{{g}}
\def\evh{{h}}
\def\evi{{i}}
\def\evj{{j}}
\def\evk{{k}}
\def\evl{{l}}
\def\evm{{m}}
\def\evn{{n}}
\def\evo{{o}}
\def\evp{{p}}
\def\evq{{q}}
\def\evr{{r}}
\def\evs{{s}}
\def\evt{{t}}
\def\evu{{u}}
\def\evv{{v}}
\def\evw{{w}}
\def\evx{{x}}
\def\evy{{y}}
\def\evz{{z}}

\def\mA{{\bm{A}}}
\def\mB{{\bm{B}}}
\def\mC{{\bm{C}}}
\def\mD{{\bm{D}}}
\def\mE{{\bm{E}}}
\def\mF{{\bm{F}}}
\def\mG{{\bm{G}}}
\def\mH{{\bm{H}}}
\def\mI{{\bm{I}}}
\def\mJ{{\bm{J}}}
\def\mK{{\bm{K}}}
\def\mL{{\bm{L}}}
\def\mM{{\bm{M}}}
\def\mN{{\bm{N}}}
\def\mO{{\bm{O}}}
\def\mP{{\bm{P}}}
\def\mQ{{\bm{Q}}}
\def\mR{{\bm{R}}}
\def\mS{{\bm{S}}}
\def\mT{{\bm{T}}}
\def\mU{{\bm{U}}}
\def\mV{{\bm{V}}}
\def\mW{{\bm{W}}}
\def\mX{{\bm{X}}}
\def\mY{{\bm{Y}}}
\def\mZ{{\bm{Z}}}
\def\mBeta{{\bm{\beta}}}
\def\mPhi{{\bm{\Phi}}}
\def\mLambda{{\bm{\Lambda}}}
\def\mSigma{{\bm{\Sigma}}}

\newcommand{\tens}[1]{\bm{\mathsfit{#1}}}
\def\tA{{\tens{A}}}
\def\tB{{\tens{B}}}
\def\tC{{\tens{C}}}
\def\tD{{\tens{D}}}
\def\tE{{\tens{E}}}
\def\tF{{\tens{F}}}
\def\tG{{\tens{G}}}
\def\tH{{\tens{H}}}
\def\tI{{\tens{I}}}
\def\tJ{{\tens{J}}}
\def\tK{{\tens{K}}}
\def\tL{{\tens{L}}}
\def\tM{{\tens{M}}}
\def\tN{{\tens{N}}}
\def\tO{{\tens{O}}}
\def\tP{{\tens{P}}}
\def\tQ{{\tens{Q}}}
\def\tR{{\tens{R}}}
\def\tS{{\tens{S}}}
\def\tT{{\tens{T}}}
\def\tU{{\tens{U}}}
\def\tV{{\tens{V}}}
\def\tW{{\tens{W}}}
\def\tX{{\tens{X}}}
\def\tY{{\tens{Y}}}
\def\tZ{{\tens{Z}}}

\def\gA{{\mathcal{A}}}
\def\gB{{\mathcal{B}}}
\def\gC{{\mathcal{C}}}
\def\gD{{\mathcal{D}}}
\def\gE{{\mathcal{E}}}
\def\gF{{\mathcal{F}}}
\def\gG{{\mathcal{G}}}
\def\gH{{\mathcal{H}}}
\def\gI{{\mathcal{I}}}
\def\gJ{{\mathcal{J}}}
\def\gK{{\mathcal{K}}}
\def\gL{{\mathcal{L}}}
\def\gM{{\mathcal{M}}}
\def\gN{{\mathcal{N}}}
\def\gO{{\mathcal{O}}}
\def\gP{{\mathcal{P}}}
\def\gQ{{\mathcal{Q}}}
\def\gR{{\mathcal{R}}}
\def\gS{{\mathcal{S}}}
\def\gT{{\mathcal{T}}}
\def\gU{{\mathcal{U}}}
\def\gV{{\mathcal{V}}}
\def\gW{{\mathcal{W}}}
\def\gX{{\mathcal{X}}}
\def\gY{{\mathcal{Y}}}
\def\gZ{{\mathcal{Z}}}

\def\sA{{\mathbb{A}}}
\def\sB{{\mathbb{B}}}
\def\sC{{\mathbb{C}}}
\def\sD{{\mathbb{D}}}
\def\sF{{\mathbb{F}}}
\def\sG{{\mathbb{G}}}
\def\sH{{\mathbb{H}}}
\def\sI{{\mathbb{I}}}
\def\sJ{{\mathbb{J}}}
\def\sK{{\mathbb{K}}}
\def\sL{{\mathbb{L}}}
\def\sM{{\mathbb{M}}}
\def\sN{{\mathbb{N}}}
\def\sO{{\mathbb{O}}}
\def\sP{{\mathbb{P}}}
\def\sQ{{\mathbb{Q}}}
\def\sR{{\mathbb{R}}}
\def\sS{{\mathbb{S}}}
\def\sT{{\mathbb{T}}}
\def\sU{{\mathbb{U}}}
\def\sV{{\mathbb{V}}}
\def\sW{{\mathbb{W}}}
\def\sX{{\mathbb{X}}}
\def\sY{{\mathbb{Y}}}
\def\sZ{{\mathbb{Z}}}

\def\emLambda{{\Lambda}}
\def\emA{{A}}
\def\emB{{B}}
\def\emC{{C}}
\def\emD{{D}}
\def\emE{{E}}
\def\emF{{F}}
\def\emG{{G}}
\def\emH{{H}}
\def\emI{{I}}
\def\emJ{{J}}
\def\emK{{K}}
\def\emL{{L}}
\def\emM{{M}}
\def\emN{{N}}
\def\emO{{O}}
\def\emP{{P}}
\def\emQ{{Q}}
\def\emR{{R}}
\def\emS{{S}}
\def\emT{{T}}
\def\emU{{U}}
\def\emV{{V}}
\def\emW{{W}}
\def\emX{{X}}
\def\emY{{Y}}
\def\emZ{{Z}}
\def\emSigma{{\Sigma}}

\newcommand{\etens}[1]{\mathsfit{#1}}
\def\etLambda{{\etens{\Lambda}}}
\def\etA{{\etens{A}}}
\def\etB{{\etens{B}}}
\def\etC{{\etens{C}}}
\def\etD{{\etens{D}}}
\def\etE{{\etens{E}}}
\def\etF{{\etens{F}}}
\def\etG{{\etens{G}}}
\def\etH{{\etens{H}}}
\def\etI{{\etens{I}}}
\def\etJ{{\etens{J}}}
\def\etK{{\etens{K}}}
\def\etL{{\etens{L}}}
\def\etM{{\etens{M}}}
\def\etN{{\etens{N}}}
\def\etO{{\etens{O}}}
\def\etP{{\etens{P}}}
\def\etQ{{\etens{Q}}}
\def\etR{{\etens{R}}}
\def\etS{{\etens{S}}}
\def\etT{{\etens{T}}}
\def\etU{{\etens{U}}}
\def\etV{{\etens{V}}}
\def\etW{{\etens{W}}}
\def\etX{{\etens{X}}}
\def\etY{{\etens{Y}}}
\def\etZ{{\etens{Z}}}

\newcommand{\pdata}{p_{\rm{data}}}
\newcommand{\ptrain}{\hat{p}_{\rm{data}}}
\newcommand{\Ptrain}{\hat{P}_{\rm{data}}}
\newcommand{\pmodel}{p_{\rm{model}}}
\newcommand{\Pmodel}{P_{\rm{model}}}
\newcommand{\ptildemodel}{\tilde{p}_{\rm{model}}}
\newcommand{\pencode}{p_{\rm{encoder}}}
\newcommand{\pdecode}{p_{\rm{decoder}}}
\newcommand{\precons}{p_{\rm{reconstruct}}}

\newcommand{\laplace}{\mathrm{Laplace}} 

\newcommand{\E}{\mathbb{E}}
\newcommand{\Ls}{\mathcal{L}}
\newcommand{\R}{\mathbb{R}}
\newcommand{\emp}{\tilde{p}}
\newcommand{\lr}{\alpha}
\newcommand{\reg}{\lambda}
\newcommand{\rect}{\mathrm{rectifier}}
\newcommand{\softmax}{\mathrm{softmax}}
\newcommand{\sigmoid}{\sigma}
\newcommand{\softplus}{\zeta}
\newcommand{\KL}{D_{\mathrm{KL}}}
\newcommand{\standarderror}{\mathrm{SE}}
\newcommand{\normlzero}{L^0}
\newcommand{\normlone}{L^1}
\newcommand{\normltwo}{L^2}
\newcommand{\normlp}{L^p}
\newcommand{\normmax}{L^\infty}

\newcommand{\parents}{Pa} 

\let\ab\allowbreak

 \newcommand{\todo}[1]{{\color{red}\noindent[ToDos: #1]}}

\newcommand{\linit}{\ell_{init}}
\newcommand{\fan}[1]{\textcolor{blue}{#1}}
\newcommand{\anil}[1]{\textcolor{red}{#1}}

\newenvironment{proofof}[1]{\begin{proof}[Proof of #1]}{\end{proof}}
\newenvironment{proofsketch}{\begin{proof}[Proof Sketch]}{\end{proof}}
\newenvironment{proofsketchof}[1]{\begin{proof}[Proof Sketch of #1]}{\end{proof}}

\def\pr{\qopname\relax n{Pr}}
\def\ex{\qopname\relax n{E}}
\def\min{\qopname\relax n{min}}
\def\max2{\qopname\relax n{max2}}
\def\max{\qopname\relax n{max}}
\def\argmin{\qopname\relax n{argmin}}
\def\argmax{\qopname\relax n{argmax}}
\def\avg{\qopname\relax n{avg}}

\def\Pr{\qopname\relax n{\mathbf{Pr}}}
\def\Ex{\qopname\relax n{\mathbf{E}}}

\newcommand{\RR}{\mathbb{R}}
\newcommand{\NN}{\mathbb{N}}
\newcommand{\ZZ}{\mathbb{Z}}
\newcommand{\QQ}{\mathbb{Q}}
\newcommand{\II}{\mathbb{I}}

\def\A{\mathcal{A}}
\def\B{\mathcal{B}}
\def\C{\mathcal{C}}
\def\D{\mathcal{D}}
\def\E{\mathcal{E}}
\def\F{\mathcal{F}}
\def\G{\mathcal{G}}
\def\H{\mathcal{H}}
\def\I{\mathcal{I}}
\def\J{\mathcal{J}}
\def\L{\mathcal{L}}
\def\M{\mathcal{M}}
\def\P{\mathcal{P}}
\def\R{\mathcal{R}}
\def\S{\mathcal{S}}
\def\O{\mathcal{O}}
\def\T{\mathcal{T}}
\def\V{\mathcal{V}}
\def\X{\mathcal{X}}
\def\Y{\mathcal{Y}}

\def\eps{\epsilon}

\def \cD {\mathcal{D}}
\def \cG {\mathcal{G}}
\def \cH {\mathcal{H}}
\def \cR {R}
\def \cX {\mathcal{X}}
\def \cY {\mathcal{Y}}

\def\x{\bm{x}} 
\def\XX{\textbf{X}} 
\def\y{\bm{y}} 
\def\w{\bm{w}} 
\def\c{\bm{c}} 
\def\e{\bm{e}} 
\def\r{\bm{r}} 
\def\v{\bm{v}} 
\def\z{\bm{z}} 
\def\p{\bm{p}} 
\def\q{\bm{q}} 
\def\u{\bm{u}} 
\def\v{\bm{v}}

\newcommand{\mini}[1]{\mbox{minimize} & {#1} &\\}
\newcommand{\maxi}[1]{\mbox{maximize} & {#1 } & \\}
\newcommand{\maxis}[1]{\mbox{max} & {#1 } & \\}
\newcommand{\minis}[1]{\mbox{min} & {#1 } & \\}
\newcommand{\find}[1]{\mbox{find} & {#1 } & \\}
\newcommand{\stt}{\mbox{subject to} }
\newcommand{\sts}{\mbox{s.t.} }
\newcommand{\con}[1]{&#1 & \\}
\newcommand{\qcon}[2]{&#1, & \mbox{for } #2.  \\}
\newenvironment{lp}{\begin{equation}  \begin{array}{lll}}{\end{array}\end{equation} }
\newenvironment{lp*}{\begin{equation*}  \begin{array}{lll}}{\end{array}\end{equation*}}

\newcommand{\probdet}{\textsc{StraC}}
\newcommand{\prob}{\textsc{StraC}$\langle \cH, R, c  \rangle$} 
\newcommand{\probL}{\textsc{StraC}$\langle \cH_d, R, c  \rangle$} 
\newcommand{\probrand}{\textsc{RandSC}}

\newcommand{\blue}[1]{\textcolor{blue}{#1}}

\title{Learning the Optimal Recommendation from Explorative Users}

\author{\name Fan Yao$^1$ \email fy4bc@virginia.edu
       \AND
       \name Chuanhao Li$^1$ \email cl5ev@virginia.edu 
       \AND
       \name Denis Nekipelov$^2$ \email dn4w@virginia.edu 
       \AND
       \name Hongning Wang$^1$ \email hw5x@virginia.edu 
       \AND
       \name Haifeng Xu$^1$ \email hx4ad@virginia.edu \\ \\ 
       \addr $^1$Department of Computer Science, University of Virginia, USA \\
       \addr $^2$Department of Economics, University of Virginia, USA
    }


\maketitle

\begin{abstract}
We propose a new problem setting to study the sequential interactions between a recommender system and a user.  Instead of assuming the user is omniscient, static, and explicit, as the classical practice does, we sketch a more realistic user behavior model, under which the user: 1) rejects recommendations if they are clearly worse than others; 2) updates her utility estimation based on rewards from her accepted recommendations; 3) withholds realized rewards from the system. We formulate the interactions between the system and such an explorative user in a $K$-armed bandit framework and study the problem of learning the optimal recommendation on the system side. We show that efficient system learning is still possible but is more difficult. In particular, the system can identify the best arm with probability at least $1-\delta$ within $O(1/\delta)$ interactions, and we prove this is tight. Our finding contrasts the result for the problem of best arm identification with fixed confidence, in which the best arm can be identified with probability $1-\delta$ within $O(\log(1/\delta))$ interactions. This gap illustrates the inevitable cost the system has to pay when it learns from an explorative user's revealed preferences on its recommendations rather than from the realized rewards. 
\end{abstract}

\section{Introduction}

Recommender systems (RS) are typically built on the interactions among three parties: the system, the users, and the items \cite{tennenholtz2019rethinking}. By collecting user-item interactions, the system aims to predict a user's preference over items. This type of problem setting has been extensively studied for decades and has seen many successes in various real-world applications, such as content recommendation, online advertising, and e-commerce platforms \cite{das2007google,linden2003amazon,koren2009matrix,schafer1999recommender,gopinath2011personalized}. 
 
Most previous works have modeled the RS users as an unknown but omniscient ``\emph{classifier}'' \cite{das2007google,li2010contextual,linden2003amazon} that allows the system to query their preference over the candidate items directly. However, for at least two reasons, such modeling assumptions of static and omniscient users appear less realistic in many modern RS applications. First, given the huge size of candidate choices, a typical user is usually not fully aware of her ``true preference'' but needs to estimate it via the interactions with the RS. For instance, an ordinary user on video-sharing apps like TikTok or review-sharing apps like Yelp does not have pre-determined rewards on all possible choices or know the optimal choice in advance. Instead, she has to consume the recommendations in order to discover the desirable content gradually. Second, in many applications, users simply respond to recommendations with accept/reject decisions rather than reveal their consumed items' utility. This situation is reflected in most practical recommendation systems nowadays. Platforms like TikTok and Yelp can easily collect binary feedback like click/non-click while struggling to evaluate the users' actual extent of satisfaction \cite{schnabel2018short}. These two limitations challenge the reliability of existing recommendation solutions in utility estimation from user feedback and thus shake the foundation of modern recommender systems. 
 
To address the challenges mentioned above, we introduce a more realistic user behavior model,  which hinges on two assumptions of today's RS users. Firstly, we believe the users are also learning the items' utilities via exploration. Their feedback becomes more relevant to the item's utility only after gaining more experience, e.g., consuming the recommended item. This perspective has been observed and supported in numerous cognitive science \cite{cohen2007should,daw2006cortical}, behavior science \cite{gershman2018deconstructing,wilson2014humans}, and marketing science studies \cite{villas2004consumer}. For instance, \citet{zhang2013forgetful} showed through a multi-armed bandit experiment that humans maintain confidence levels regarding different choices and eliminate sub-optimal choices to achieve long-term goals when they are aware of the uncertain environment. These works motivate us to study the recommendation problem under a more realistic user model, where the user keeps refining her assessment about an item after consuming it, and she is willing to explore the uncertainty when deciding on the recommendations. Formally, we model her exploration as being driven by her estimated confidence intervals of each item's reward: she will only reject an item when its estimated reward is clearly worse than others.
 
Secondly, we assume that only the users' \emph{binary} responses of acceptance are revealed to the system, whereas the realized user reward of any consumed items is kept only to the user (to improve her own reward estimation).  This thus gives rise to an intriguing challenge of learning user's utility parameters from only the coarse and implicit feedback of ``revealed preference'' \cite{richter1966revealed}. One may naturally wonder why the user does not simply give all her realized rewards to the system since both the system and user are learning the best recommendation for the user. This is certainly an ideal situation but, unfortunately, is highly unrealistic in practice. As we mentioned before, it is widely observed that very few RS users would bother to provide detailed feedback (even not numerical ratings). This observation is also supported by the 90-9-1 Rule for online community engagement, and the ``Lazy User Theory" \cite{tetard2009lazy} in the HCI community, which states that a user will most often choose the solution that will fulfill her information needs with the least effort.
 
Under this more realistic yet challenging environment, it is unclear whether efficient system learning is even possible, i.e., can the system still discover the user's real preference? To answer the question, we formulate the interactions between the system and such an explorative user in a $K$-armed bandit framework and study best arm identification (BAI) with fixed confidence. We design an efficient online learning algorithm and prove it obtains an $O(\frac{1}{\delta})$ upper bound on the number of recommendations to identify the best arm with probability at least $1-\delta$. We also show that this bound is tight by proving a matching lower bound for any successful algorithm. Our results illustrate the inevitable gap between the performance under the standard best arm identification setting and our setting, which indicates the intrinsic hardness in learning from an explorative user's revealed preferences. Our experiments also demonstrate the encouraging performance of our algorithms compared to the state-of-the-art algorithms for BAI applied to our learning setup.
\\paragraph{Related Work}

The first related direction to this work is the problem of best arm identification (BAI) with fixed confidence \cite{garivier2016optimal}. Instead of minimizing regret, the system aims to find the arm with the highest expected reward with probability $1-\delta$, while minimizing the number of pulls $T$. The tight instance-dependent upper bound for $T$ is known to be $O(H\log \frac{1}{\delta})$ \cite{carpentier2016tight}, where $H$ is a constant describing the intrinsic hardness of the problem instance. In our work, the system shares the same goal but under a set of more challenging restrictions posed by learning from an explorative user's revealed preferences. For example, the system cannot directly observe the realized reward of each pulled arm. We prove that under this new learning setup, the budget upper bound increases from $O(\log\frac{1}{\delta})$ to $O(\frac{1}{\delta})$. 

There are also previous works that focus on online learning without access to the actual rewards. The dueling bandit problem proposed in \cite{yue2012k} modeled partial-information feedback where actions are restricted to noisy comparisons between pairs of arms. Our feedback assumption is more challenging than dueling bandit in two aspects. First, we do not assume the system knows the reference arm to which the user compares when making her decisions. Second, the user's feedback is evolving over time as she learns from the realized rewards. Hence, none of existing dueling bandit algorithms \cite{yue2011beat,komiyama2015regret,zoghi2014relative} can address our problem. The unobserved reward setting is also studied in inverse reinforcement learning. For instance, Hoiles et al. \cite{hoiles2020rationally} used Bayesian revealed preferences to study if there is a utility function that rationalizes observed user behaviors. Their work focused on user behavior modeling itself while we studied system learning and analyzed the outcome induced by this type of user behavior assumption.

Another remotely related direction is incentivized exploration in the Internet economy. In such a problem, a system aims to maximize the welfare of a group of users who only care about their short-term utility. \citet{kremer2014implementing} first studied this problem and developed a policy that attains optimal welfare by partially disclosing information to different users. Follow-up works extended the setting by allowing users to communicate \cite{bahar2015economic} and introducing incentive-compatibility constraints \cite{mansour2016bayesian,mansour2020bayesian}. Our motivation considerably differs from this line of work in three important aspects. First, incentivized exploration looks at an informationally advantaged principal whereas our system is in an informational disadvantageous position, as it has mere access to the user's revealed preferences. Second, their setting looks at how to influence user decisions through signaling with misaligned incentives, whereas we are trying to help a boundedly rational ordinary user to learn their best action in a cooperative environment. Third, the user in our model is an adaptive learning agent rather than a one-time visitor to the system.

\section{Modeling Users' Revealed Preferences}\label{sec:model}


To study the sequential interactions between a recommender system and an explorative user, we adopt a stochastic bandit framework where the time step $t$ is used to index each interaction, and the set of arms $[K]=\{1,\cdots,K\}$ denote the recommendation candidates. At each time step $t$, the following events happen in order:
\begin{enumerate}
    \item The system recommends an arm $a_t$ to the user;
    \item The user decides whether to accept or reject $a_t$. If the user accepts $a_t$, realized reward $r_{a_t, t}$ is disclosed to the user afterwards.
    \item The system observes the user's binary   decision of acceptance or not, i.e., the   revealed preference \cite{richter1966revealed}.  
\end{enumerate}

From the user's perspective, we denote the true reward of each arm $i \in [K]$ as $\mu_i$, and the realized reward after each acceptance of arm $i$ is drawn independently from a sub-Gaussian distribution with mean $\mu_i$ and unit variance. 
The system has a long-term objective and aims to find the best arm while minimizing the total number of recommendations. This renders our problem   a best arm identification (BAI) problem with fixed confidence but based on partial information about the rewards.
Throughout the paper, we assume without loss of generality that $\mu_*=\mu_1>\mu_2\geq \mu_3\geq \cdots \geq \mu_K$, and let $\Delta_1=\mu_1-\mu_2, \Delta_i = \mu_{*} - \mu_i > 0, \forall i >1.$ Following the convention in BAI literature \cite{carpentier2016tight,audibert2010best}, we further define the quantity $H=\sum_{i=1}^K \frac{1}{\Delta_i^2}$ which characterizes the hardness of the problem instance. 

As discussed previously, the user cannot choose from the full arm set but can only decide whether to accept or reject the recommended arms from the system. To make a decision at time $t$, we assume the user utilizes the information in all previous interactions by maintaining a confidence interval $CI_{i,t}=(lcb_{i,t},ucb_{i,t})$ for each arm $i$, which is defined as 
\begin{equation*}
    (lcb_{i,t},ucb_{i,t})\triangleq\left(\hat{\mu}_{i,t}-\sqrt{\frac{\Gamma(t;\rho,\alpha)}{n_i^t}},\hat{\mu}_{i,t}+\sqrt{\frac{\Gamma(t;\rho,\alpha)}{n_i^t}}\right),
\end{equation*} 
where $lcb$ and $ucb$ stand for the lower/upper confidence bounds respectively, $n_i^t$ is the number of 
\emph{acceptances} on arm-$i$ up to time $t$, $\hat{\mu}_{i,t}=\frac{1}{n_i^t}\sum_{s=1}^t \mathbb{I}[i \text{~is~accepted~at~} s]r_{i,s}$ is the empirical mean reward of arm $i$ at time $t$, and $\Gamma(t;\rho,\alpha)$ is a function parameterized by $\{\rho_t,\alpha\}$, which characterize the user's confidence level to her reward estimation at time step $t$. 
Following the convention rooted in the UCB1 algorithm \cite{auer2002finite}, we consider the (flexible) confidence bound form:
\begin{equation}\label{eq:trust}
    \Gamma(t;\rho,\alpha)=\max\left\{0, 2\alpha\log [\rho_t \cdot n(t)]\right\},
\end{equation}
where $n(t)=\sum_{i\in [K]}n_i^t$ is the total number of accepted recommendations up to time $t$. We note that the choice of $\Gamma$ is to   flexibly cover possibly varied user types  captured by parameters $\alpha$ and $\rho_t$. In particular, $\alpha$ directly controls the span of the CIs and thus represents the user's intrinsic tendency to explore: a larger $\alpha$ indicates a higher tolerance for the past observations, meaning the user is more willing to accept recommendations in a wider range. $\rho_t: \mathbb{N} \xrightarrow{} [\rho_0, \rho_1]$ is allowed to be any sequence that can depend on the interaction history and has a bounded range $[\rho_0, \rho_1] \subset (0, +\infty)$, which captures the cases where the user's confidence over the system evolves over time. For example, $\rho_t$ can be a function of the acceptance rate $\frac{n(t)}{t} \in [0, 1]$ and increases monotonically with respect to $\frac{n(t)}{t}$. Our results only rely on the lower and upper bound of $\rho_t$ and are oblivious to its specific format. Note that for the special case of $\alpha=1, \rho_t=1 \,  \forall t$, Eq \eqref{eq:trust} corresponds to 
the classic confidence interval defined in UCB1. We remark that parameters $\alpha$ and $\rho_t$ are     only to characterize different types of users, which provide flexibility in handling different real-world scenarios; but they are \emph{not} introduced for our solution.

\noindent
\textbf{The decision rule.} 
When an arm $i$ is recommended, we assume the user will reject it if and only if   there exists $j \neq i$ such that $lcb_j \geq ucb_i$. That is, the user only accepts an arm if  there is no other arm that is clearly better than the recommended one with a high confidence. 
The rationale behind our imposed user decision rule is straightforward: first, the user should not miss the arm with the highest lower confidence bound as this is arguably the safest choice for the user at the moment; second, if two arms have chained confidence intervals, the user does not have enough information to distinguish which one is better, and hence should not reject either one, i.e., being explorative. 


\section{Learning from Explorative Users' Revealed Preferences}\label{sec:algo}

With stochastic rewards, we know $\mathbb{P}[\mu_i \in CI_{i,t}]$ almost always increases as the number of acceptances $n(t)$ grows. Therefore, the system can confidently  rule out a sub-optimal arm once it has collected a reasonable number of acceptances. In light of this, we devise a two-phase explore-then-exploit strategy for system learning: the system first accumulates a sufficient number of acceptances and then examines through the arm set to eliminate sub-optimal ones with a high confidence. 

\begin{algorithm}[t]
   \caption{Phase-1 Sweeping}
   \label{alg:phase1_alg}
\begin{algorithmic}
   \STATE {\bfseries Input:} $K>0$, $\delta\in (0,1) ,N_1(\delta)>0$. 
\STATE {\bfseries Initialization:} $F=[K]$,$ N=0, n_i=0, i\in[K].$

\REPEAT
    \STATE Recommend each item in $F$ once, and remove the rejected ones from $F$.
\UNTIL{$F$ is empty or the time step exceeds $N_1(\delta)$.}

\REPEAT
\STATE If $F$ is empty, reset $F= \{1,\cdots, K\};$
    \FOR{$i \in F$} 
    \STATE Recommend $i$ until rejected, then remove it from $F$.
\ENDFOR
\UNTIL{The time step exceeds $N_1(\delta)$.}

   \STATE {\bfseries Output:} number of acceptances for each arm $\{n_i\}_{i=1}^K$. 
\end{algorithmic}
\end{algorithm}

\begin{algorithm}[t]
   \caption{Phase-2 Elimination}
   \label{alg:phase2_alg}
\begin{algorithmic}
   \STATE {\bfseries Input:} $K>0$, $\{n_i\}_{i=1}^K$ from Phase-1.
   \STATE {\bfseries Initialization:} $F=[K] $.
   \
   
   \WHILE{$|F|>1$}
   \STATE Recommend $a_t=\arg\min_{i\in[K]} n_i$ and update $n_{a_t}$.
   \STATE Remove $a_t$ from $F$ if rejected.
   \ENDWHILE
 
   \STATE {\bfseries Output:} $F$.
\end{algorithmic}
\end{algorithm}


The Phase-1 design is presented in Algorithm \ref{alg:phase1_alg}. Like standard bandit algorithms,  
the system will execute an initialization procedure by sweeping through the arm set $F=[K]$ and then recommend each arm repeatedly until it collects exactly one rejection on each arm there.  This initialization stage is similar to the round-robin style pulls in  standard bandit algorithms (e.g., UCB1, $\epsilon$-Greedy). But the key difference is that our algorithm will initialize by collecting one rejection on each arm whereas standard bandit algorithm will initialize by collecting one pull (i.e., acceptance) on each arm. This is because rejections in our setup are more informative than acceptances to the system. 
After initialization, Algorithm \ref{alg:phase1_alg} enters the main loop and do the following: keeps recommending the same arm until it gets rejected and then moves to another arm in $F$. After each arm gets rejected once, the system starts a new round by resetting $F=[K]$. This procedure continues until the total number of acceptances exceeds $N_1(\delta)$. 
This sweeping strategy reflects the principle of Phase-1: the system aims to collect a reasonable number of acceptances while minimizing the number of rejections by not recommending any risky arm. 
For the ease of analysis, we divide Phase-1 into different \emph{rounds} (indexed by $r$) by the time steps when the system resets $F$.  We will prove later that there is a tailored choice of $N_1(\delta)$ such that when the system enters Phase-2 with $N_1(\delta)$ acceptance, it can identify the best arm with probability $1-\delta$.

One might notice that the Phase-1 algorithm \ref{alg:phase1_alg} needs to recommend each arm at least once and also needs to recommend each arm repeatedly in the candidate set to the user. Now a natural question to ask is, how this design could be practical given the immense size of item pool in a practical recommender system. However, we note that an arm in our model can be viewed as a type/category of items, rather than just literally an individual item. Therefore, Algorithm \ref{alg:phase1_alg} should not be interpreted as recommending the same item repeatedly to a user, but instead recommending items from the same category or type. This is also the typical interpretation of arms in stochastic bandit literature. Moreover, consuming repeatedly each item in the candidate set is a typical requirement in the stochastic bandit problems, due to the observation noise. More specifically, under a stochastic reward setting, the realized reward at each time step is randomly drawn from an underlying reward distribution, and repeated interactions are necessary to pin down the distribution. 

We now present the design for Phase-2, as shown in Algorithm \ref{alg:phase2_alg}, which follows Phase-1 Sweeping. Here, the system executes arm elimination: always recommend the arm with the minimum number of acceptances; and eliminate an arm when it is rejected by the user, until there is only one arm left. We prove that the stopping time of Algorithm \ref{alg:phase2_alg} is finite with probability 1, and it outputs the best arm with probability $1-\delta$ when it terminates. We name our proposed two-phase algorithm  \textbf{B}est \textbf{A}rm \textbf{I}dentification under \textbf{R}evealed preferences, or BAIR for short. 

Next, we analyze BAIR by upper bounding its stopping time given fixed confidence $\delta$. Our main result is formalized in the following theorem.

\begin{theorem}
\label{thm:upperbound2}
When $\Gamma$ is defined as in Eq \eqref{eq:trust}, with probability at least $1-2\delta$, the system   makes at most $$  O\Big(K^{\frac{1}{\alpha}}\delta^{-\frac{1}{\alpha}}+K^{1+\frac{1}{2\alpha}}\delta^{-\frac{1}{2\alpha}}\sqrt{\log\frac{K}{\delta}}+\frac{\alpha K}{\Delta_1^2}\log\frac{ K}{\delta \Delta_1}\Big)$$ recommendations and successfully identifies the best arm by running Algorithm \ref{alg:phase1_alg} and \ref{alg:phase2_alg}. 
\end{theorem}
Note that the upper bound on the number of rounds above is deterministic while not in expectation. The proof of Theorem \ref{thm:upperbound2} requires separate analysis for Phase-1 and Phase-2, which we discuss in the following subsections. At a high level, the first two terms in the bound come from the number of acceptances and rejections in Phase-1, and the last term corresponds to the number of acceptances in Phase-2. We decompose the bound in Theorem \ref{thm:upperbound2} in Table \ref{tb:upperbounds}.
\begin{table}[ht]
\vspace{-1mm}
\centering
\caption{Upper bounds on \#recommendations in Phase 1,2}
\begin{tabular}{c|c|c|c}
 \hline
 Phase & \# Acceptance & \# Rejection &  Prob. \\  \hline
 1 & $O(K^{\frac{1}{\alpha}}\delta^{-\frac{1}{\alpha}})$ & $O(K^{\frac{1+2\alpha}{2\alpha}}\delta^{-\frac{1}{2\alpha}}\log^{\frac{1}{2}}\frac{K}{\delta})$ & $1-\delta$ \\  \hline
 2 & $O(\frac{\alpha K}{\Delta_1^2}\log\frac{K}{\delta \Delta_1})$ & $K$ & $1-\delta$ \\ \hline
\end{tabular}
\vspace{-2mm}
\label{tb:upperbounds}
\end{table}

Note that there is a clear \emph{tradeoff} between the upper bounds in Phase-1 and Phase-2 in terms of $\alpha$: a smaller $\alpha$ increases the upper bound in Phase-1 but requires less number of recommendations in Phase-2, while a larger $\alpha$ ensures a lighter Phase-1 but would result in a more cumbersome Phase-2. This is expected because, e.g., when facing a highly explorative user (large $\alpha$), the system can easily accumulate sufficient acceptances in Phase-1. However, it will need more comparisons in Phase-2 to identify the best arm for such a highly explorative user. Theoretically, there exists an optimal $\alpha$ which minimizes the total number of recommendations; however, this is not particularly interesting to investigate in this paper, as $\alpha$ is not under the system's control, but a characterization of the user.



\subsection{Upper Bound for Phase-2}\label{subsec:ubphase2}
We start the analysis for Phase-2 first as it will lead to the correct $N_1$ for us to run Phase-1. Specifically, we prove that when $N_1=\frac{1}{\rho_0}\cdot\Big(\frac{2K}{\delta}\Big)^{\frac{1}{\alpha}}$ acceptances are accumulated in Phase-1, it is safe for the system to move on to Phase-2.
\begin{lemma}\label{lm:phase2bound1}
If Phase-1 terminates with $N_1=\frac{1}{\rho_0}\cdot\Big(\frac{2K}{\delta}\Big)^{\frac{1}{\alpha}}$ acceptances,  the Phase-2 Algorithm \ref{alg:phase2_alg}  will output the best arm with probability at least $1-\delta$.
\end{lemma}

The next Lemma shows that no matter when the system enters Phase-2, Algorithm \ref{alg:phase2_alg} must terminate with probability $1-\delta$ within $O(\log{\frac{1}{\delta}})$ steps.
\begin{lemma}\label{lm:phase2bound2}
With probability $1-\delta$, Algorithm \ref{alg:phase2_alg} terminates within $O(K+\sum_{i=1}^K \frac{\alpha}{\Delta_i^2}\log \frac{\rho_1K}{\rho_0\delta\Delta_i})$ steps.
\end{lemma}
The first term $O(K)$ in the bound corresponds to the number of rejections in Phase-2, since Phase-2 Elimination incurs at most $K-1$ rejections by definition. The second term characterizes the number of acceptances, which matches the tight lower bound of BAI with fixed budget \cite{carpentier2016tight} in terms of $\delta$ with a factor $\sum_{i=1}^K {\frac{1}{\Delta_i^2}\log \frac{1}{\Delta_i}}$ instead of $H=\sum_{i=1}^K {\frac{1}{\Delta_i^2}}$. Thus, the bound provided by Lemma \ref{lm:phase2bound2} is almost tight. The $\rho_1$ also plays a role in the upper bound because when $\rho_1$ is too large, the user could maintain a very wide confidence interval for each arm which requires extra effort for the system to eliminate sub-optimal arms. Combining Lemma \ref{lm:phase2bound1} and Lemma \ref{lm:phase2bound2} and take $\rho_0, \rho_1$ as fixed constants, we conclude that Algorithm \ref{alg:phase2_alg} will terminate and output the best arm with probability $1-\delta$ within $O(\sum_{i=1}^K\frac{\alpha}{\Delta_i^2}\log\frac{K}{\delta\Delta_i})$ steps after Algorithm \ref{alg:phase1_alg} is equipped with $N_1(\delta)=O(K^{\frac{1}{\alpha}}\delta^{-\frac{1}{\alpha}})$. Note that compared with the theoretical guarantee for BAI with fixed confidence, our upper bound matches the lower bound in \cite{garivier2016optimal} in terms of $\delta$. This implies that once the system has accumulated sufficient acceptances in Phase-1, the learning from reveal preferences does not bring extra difficulty. However, the bottleneck for the integrated system strategy lies in Phase-1, which we now analyze. 

\subsection{Upper Bound for Rejections in Phase-1}\label{subsec:ubphase1}
 
Recall that a round in Algorithm \ref{alg:phase1_alg} is a segment of a sequence of interactions indexed by $r$, in which the candidate arm set $F$ is reset to $[K]$ at the beginning and each arm gets rejected once in the end. We abuse the notation a bit by using $[t_s^{(r)}, t_e^{(r)}]$ to denote the $r$-th round that starts from time $t_s^{(r)}$ with $N=n(t_s^{(r)})$ acceptances and ends at time $t_s^{(r)}$ with $N=n(t_e^{(r)})$ acceptances. Next we upper bound the total number of rounds in Phase-1. We prove that Algorithm \ref{alg:phase1_alg} must terminate in a small number of rounds with probability $1-\delta$, as shown in the following lemmas.

\begin{lemma}\label{lm:algo4}
For any $K>0, \delta>0, N_1>0$, with probability $1-\delta$, Algorithm \ref{alg:phase1_alg} terminates in $O(\sqrt{N_1\log \frac{K}{\delta}})$ rounds and thus incurs at most $O(K\sqrt{N_1\log \frac{K}{\delta}})$ rejections. 
\end{lemma}
In particular, if we choose $N_1 \sim O(K^{\frac{1}{\alpha}}\delta^{-\frac{1}{\alpha}})$ in accordance with Lemma \ref{lm:phase2bound1}, the total number of rejections in Phase-1 can be upper bounded by $O(K^{1+\frac{1}{2\alpha}}\delta^{-\frac{1}{2\alpha}}\sqrt{\log \frac{K}{\delta}})=o(N_1)$ as $\delta \xrightarrow{} 0$. In the next section, we will show that $N_1 \sim O(K^{\frac{1}{\alpha}}\delta^{-\frac{1}{\alpha}})$ is necessary to guarantee the success in Phase-2. The proof of Lemma \ref{lm:algo4} depends on the following two lemmas.

\begin{lemma}\cite{lattimore2020bandit}\label{lm:iteratedlog}
Let $\{X_t\}_{t=1}^{\infty}$ be a sequence of $i.i.d.$ sub-Gaussian random variables with zero mean and unit variance, and $\hat{\mu}_n=\frac{1}{n}\sum_{t=1}^n X_t$, for any $\delta\in (0,1)$, $$\mathbb{P}\Big(\forall n\in \mathbb{N}^+: |\hat{\mu}_n| \leq \sqrt{\frac{2}{n}\log \frac{n(n+1)}{\delta}} \Big) > 1-\delta.$$
\end{lemma}


\begin{lemma}\label{lm:Er}
Let $f(t) = \max_{i=1}^K \hat{\mu}_{i,t}$ be the highest empirical mean maintained by the user at time step $t$. Then for any round $r$ denoted by $[t_s^{(r)}, t_e^{(r)}]$ in Algorithm \ref{alg:phase1_alg}, we have $$f(t_e^{(r)}) \leq f(t_s^{(r)}) - 2\sqrt{\frac{\underline{\Gamma}^{(r)}}{n(t_e^{(r)})}},$$
where $\underline{\Gamma}^{(r)}=\min_{t\in[t_s^{(r)}, t_e^{(r)}]}\Gamma(t).$
\end{lemma}

Lemma \ref{lm:Er} shows an interesting property about the user's empirical reward estimation during our Algorithm \ref{alg:phase1_alg} --- the maximum empirical mean will decrease by at least $ 2\sqrt{\frac{\underline{\Gamma}^{(r)}}{n(t_e^{(r)})}}$ after each round. This implies that Phase 1 cannot run for too many rounds. Finally, assembling   Lemma \ref{lm:phase2bound1}, \ref{lm:phase2bound2}, and \ref{lm:algo4}, we can derive the upper bound for the stopping time of BAIR in Theorem \ref{thm:upperbound2}.


\subsection{The Lower Bound}\label{sec:lower}

It is worthwhile to compare our upper bound on the number of recommendations in Theorem \ref{thm:upperbound2} with the $O(\log \frac{1}{\delta})$ upper bound in standard BAI setting. Specifically, 
our bound is worse due to the leading term $O(\delta^{-\frac{1}{\alpha}})$. In this subsection, we   prove that this performance deterioration is inevitable due to our focus on an intrinsically harder setup with only the user's revealed preferences. As our second main result, the following theorem shows that the dependence of $\delta$ in the upper bound of Theorem \ref{thm:upperbound2} is tight.

\begin{theorem}\label{thm:lowerbound4delta}
For any algorithm $\pi$ and $0<c<\frac{1}{2}$, there exists a problem instance depending on $\delta$ such that if $\pi$ collects less than $N_0=\max\{\frac{\delta^{-\frac{1}{\alpha}+c}}{\rho_0}, \frac{2}{\Delta_1^2}\log\frac{1}{4\delta} \}$ accepted recommendations, it must make mistake about the best arm with probability at least $\delta$. 
\end{theorem} 
\begin{proof}[Proof Sketch]
The lower bound $N_0 \geq \frac{2}{\Delta_1^2}\log\frac{1}{4\delta}$ is  from the general lower bound result for BAI in a stochastic bandit setting, as the system in our setting can never find the best arm quicker than an BAI algorithm that has access to the realized rewards. To prove $N_0\geq \delta^{-\frac{1}{\alpha}+c}/\rho_0$, we  construct  two problem instances $\nu$ and $\nu'$ such that: 1). they have different best arms; 2). any system interacting with $\nu$ or $\nu'$ will receive exactly the same sequences of user binary responses  with probability at least $2\delta$, as long as it collects less than $N_0$ acceptances. Therefore, the system is not able to differentiate between $\nu$ and $\nu'$ with probability at least $1-\delta$, thus making mistakes about the best arm with probability $\delta$ on either $\nu$ or $\nu'$. Our final proof is based on the ensemble of the difficult instances in both situations.  We defer the detailed construction and proof to the appendix.
\end{proof}
Note that the lower bound $ \delta^{-\frac{1}{\alpha}+c}/\rho_0$ illustrates the intrinsic hardness of BAI from revealed preferences:  any algorithm has to make at least $\Omega(\delta^{-\frac{1}{\alpha}+c})$ recommendations in order to guarantee the identification of the best arm for any problem instances in our setup. This is in sharp contrast to the well-known $O(\log \frac{1}{\delta})$ lower bound in the standard bandit reward feedback  setting. 
\section{Experiments}
In this section, we empirically study BAIR to support our theoretical analysis. We use simulations on synthetic datasets in comparison with several baseline algorithms. Since we propose a new perspective to model user-system interactions in RS, there is no baseline for direct comparison. However, this also gives us an opportunity to demonstrate how problematic it may be when using a wrong user model for the observed system-user interactions. 

\begin{table*}[t]
\centering
\begin{small}
\caption{Comparison between BAIR and three baselines on proposed metrics. $(\alpha=1)$}
\label{tb:exper-baselines}
\begin{tabular}{ll|cc|cccc|cccc}
\toprule
 &   & \multicolumn{2}{c}{Stopping Time} & \multicolumn{4}{c}{Rejection Rate (\%)} & \multicolumn{4}{c}{Prob. of Success} \\
$\delta$ & $K$ & BAIR & T\&S & BAIR & UNI& EXP3 & T\&S & BAIR & UNI & EXP3 & T\&S \\
\midrule

\multirow{4}{*}{$0.1$} & 2    & 405  & 539 & 0.8 & 8.5  &  2.5 & 1.0 & 0.999  & 0.786  & 0.621  & 0.999 \\
& 5                           & 737  & 1069 & 1.4 & 27.6 & 8.8  & 1.7 & 1.000  & 0.568  & 0.549  & 0.971 \\
& 20                          & 2113 & 3107 & 1.9 & 33.7 & 11.6 & 4.6 & 1.000 & 0.229  & 0.408  & 0.966 \\
& 100                         & 10449 & 16523& 2.0 & 34.4 & 13.4  & 11.9& 1.000 & 0.092 & 0.400  & 0.965\\
\hline
\multirow{4}{*}{$0.05$} & 2  & 413 & 557 & 0.7 & 9.4 & 2.8 & 1.1 & 1.000 & 0.799  & 0.654  & 1.000 \\
                        &  5 & 787 & 1123 & 1.2 & 29.2 & 7.6 & 1.9 & 1.000 & 0.599 & 0.581 & 0.978 \\
                        &  20 & 2421 & 3161 & 1.5 & 40.4 & 12.3 & 4.4 & 1.000 & 0.429& 0.577 & 0.965 \\
                        &  100 & 10826 & 16577 & 2.2 & 43.2 & 14.1 & 11.1 & 1.000 & 0.150& 0.534  & 0.945 \\
\hline
\multirow{4}{*}{$0.02$} & 2  & 422 & 544 & 0.7 & 12.2 & 3.9 & 1.2 & 1.000 & 0.806 & 0.638  & 0.999 \\
                        &  5 & 879 & 1148 & 1.2 & 31.8 & 7.7 & 2.1 & 1.000 & 0.670 & 0.607  & 0.962 \\
                        &  20 & 3593 & 3210 & 1.6 & 50.2 & 11.9 & 4.4 & 1.000 & 0.436& 0.680& 0.957 \\
                        &  100 & 11528 & 16955 & 2.5 & 46.7 & 14.2 & 12.3 & 1.000 & 0.153 & 0.634 & 0.963\\
\hline
\multirow{4}{*}{$0.01$} & 2   & 437  & 554 & 0.8 & 17.6 & 2.7 & 1.3 & 1.000 & 0.821  & 0.632  & 0.998 \\
& 5                           & 940 & 1153 & 1.3 & 34.9 & 7.7  & 2.0 & 1.000 & 0.701  & 0.604  & 0.959  \\
& 20                          & 4017 & 3223& 1.4 & 51.6 & 13.6 & 3.0& 1.000 & 0.476  & 0.737  & 0.947 \\
& 100                         & 20344& 27548 & 1.5 & 52.8 & 15.9 & 12.4 & 1.000 & 0.130  & 0.725  & 0.930 \\
\hline
\multirow{4}{*}{$0.005$} & 2   & 512 & 570 & 0.8 & 27.6 & 4.0 & 1.4 & 1.000 & 0.992  & 0.901  & 1.000 \\
& 5                            & 1692 & 1580 & 0.8 & 50.0 & 8.3 & 1.6 & 1.000 & 0.942   & 0.844  & 0.993  \\
& 20                           & 7827 & 5983& 0.7 & 71.4 & 12.4 & 5.4 & 1.000 & 0.938  & 0.921  & 0.979 \\
& 100                         & 40246 & 37931 & 0.7 & 73.1 & 15.0 & 16.0 & 1.000 & 0.794  & 0.930  & 0.970 \\
\bottomrule
\end{tabular}
\end{small}
\end{table*}

\subsection{Experiment Setup and Baselines}
 
 
As we discussed in the introduction, prior works treat users as an unknown but omniscient classifier, and therefore stochastic bandits are the typical choices to learn from user feedback. Moreover, since the users' responses in our problem setup are not necessarily stochastic, adversarial bandits could be another choice. Therefore, we employ the corresponding state-of-the-art algorithms, Track-and-stop \cite{garivier2016optimal} (for BAI) and EXP3 \cite{auer2002nonstochastic} (for adversarial bandits), to compare with BAIR. Besides, we also propose a heuristic baseline, uniform exploration, to directly compete with BAIR. The details of these baselines are as follows.

    \noindent \textbf{Uniform exploration (UNI)}: The system recommends candidate arms uniformly until the number of recommendations reaches the given threshold $T$. When the algorithm terminates, it outputs the arm with the maximum number of acceptances; ties are broken arbitrarily. 
    
    \noindent \textbf{Track-and-stop (T\&S)}: 
    This is the state-of-the-art solution for BAI with fixed confidence \cite{garivier2016optimal}. The expected stopping time of T\&S  provably matches its information-theoretic lower bound $O(\log \frac{1}{\delta})$ in the stochastic bandit setting. The effectiveness of T\&S relies on the independent and stationary reward assumptions on each arm, which fail to hold in our setup as user responses are not a simple function of their received rewards. We will investigate how the theoretical optimality of the T\&S breaks down under our problem setting. 
    
    \noindent \textbf{EXP3}: 
    To the best of our knowledge, there is no BAI algorithm under an adversarial setting. As a result, we adopt EXP3 \cite{auer2002nonstochastic} for comparison. Given the number of arms $K$ and a time horizon $T>K\log K$, EXP3 is provably a no-regret learning algorithm if taking $\gamma \sim O(\sqrt{\frac{\log K}{KT}})$ and $\epsilon \sim O(\sqrt{\frac{K\log K}{T}})$. We run EXP3 with this configuration and output the arm with the maximum number of acceptances in the end.

\subsection{Simulation Environment and Metrics}
For different configurations of $(\delta, K, \Delta_1)$ for BAI, we generate 1000 independent problem instances $(\mu_i)_{i=1}^K$ by sampling each $\mu_i \in N(0,1)$ and then reset $\mu_*$ to meet the given value of $\Delta_1$. Observing that our conclusion does not vary much under different $\Delta_1$, we present the result for $\Delta_1=0.5$ in this section and leave more results in the appendix \ref{app:Delta} due to space limit. The parameters in the user model are set to $\alpha=1, \rho_t = 1+\frac{n(t)}{t} \in [1,2]$, i.e., $\rho_0=1, \rho_1=2$, and results for different choice of $\alpha$ can be found in Appendix \ref{app:additional_alpha}. We run BAIR with $N_1=\frac{2K}{\delta}$ and compare its performance with UNI, EXP3 and T\&S on the entire set of problem instances and calculate the following three metrics.

\noindent\textbf{Probability of success:} When each algorithm terminates, we examine whether the output arm is the best arm (i.e., success). The probability of success ($p$) is then given by the empirical frequency of success over all problem instances. We also calculate the value $\frac{1-p}{\delta}$ to measure if and how much the probability of success falls below the given confidence level $\delta$, which is presented right after the probability of success. 

\noindent \textbf{Rejection rate:} When each algorithm terminates at step $T$, we count the total number of rejections $\# Rej$ the system receives. The rejection ratio is given by $\frac{\# Rej}{T}$, and then averaged over all problem instances. 

\noindent \textbf{Stopping time:} It is the total number of interactions needed to terminate an algorithm. BAIR and T\&S stop by their own termination rules; UNI and EXP3 stop by the input time  $T$, since these two algorithms terminate by a preset time horizon. To make a fair comparison, we set $T$ for UNI/EXP3 as the average stopping time of BAIR under the corresponding problem instance. Hence, this metric is only set to compare BAIR and T\&S.

\subsection{Comparison Between BAIR and Baselines}

The results are reported in Table \ref{tb:exper-baselines}. Based on the comparison results for BAIR and the baselines, we have the following observations.  

\noindent\textbf{BAIR vs. T\&S.} As shown in Table \ref{tb:exper-baselines}, T\&S enjoys the best performance among three baselines on rejection rate and probability of success, but still does not work well in our problem setting. Given the confidence threshold $\delta$, T\&S fails to identify the best arm with probability $1-\delta$ for $K>2$ and $\delta<0.05$. We also find the stopping time of T\&S is worse than BAIR in most cases and fails to meet its theoretical lower bound $O(\log \frac{1}{\delta})$. This is expected: our binary user feedback cannot be simply modeled as independent and stationary rewards, which are the fundamental assumptions behind the design of T\&S. Since T\&S wrongly models user responses, it is easily misinformed by the user's potentially inaccurate feedback in the early stage. As a result, it is very likely to miss the best arm and spend most of the rest time on a wrong subset. In contrast to T\&S, BAIR is aware that the revealed preferences from the early stage are very likely to have a large variance. Therefore, it chooses to make safe recommendations at first to help the user gain more experiences (Phase-1 preparation) such that that the user will reveal more accurate feedback later on (Phase-2 elimination). This explains how BAIR achieves the goal more efficiently, even with the additional cost in Phase-1. 

\noindent\textbf{BAIR vs. UNI/EXP3.} The other two baselines, UNI and EXP3, exhibit worse performance in both the rejection ratio and the probability of success than BAIR. 
    Given the same time budget, UNI always suffers from the largest proportion of rejections because it does not take any measures to eliminate bad arms. As rejections do not update the user's empirical reward estimation, the given time budget is insufficient for UNI to differentiate the arms with similar expected rewards, thus causing a low probability of success. EXP3 enjoys a lower rejection rate than UNI, because it pulls those empirically bad arms less. The mandatory exploration in EXP3 helps correct the inaccurate early observations and gives a more competitive probability of success when $K$ gets larger. However, due to the larger variance of EXP3, if the user's estimated reward for the best arm is low at the beginning, EXP3 tends to overly focus on differentiating a group of suboptimal arms, which decreases its chance of discovering the best arm.

To summarize, the fundamental reason for the failure of these baselines lies in the insufficient system exploration when facing an explorative user. These baselines either treat the user as a black-box or assume independent and stationary user feedback, which leads to a worse empirical result in terms of both accuracy and efficiency in finding the best arm. 

The result in Table \ref{tb:exper-baselines} supports our theoretical analysis in Theorem \ref{thm:upperbound2}. When $\Delta_1=0.5$ and $\delta<0.1$, $\frac{1}{\delta}$ dominates $\frac{1}{\Delta_1^2}$ and Theorem \ref{thm:upperbound2} suggests the algorithm's stopping time grows approximately linear in $\frac{1}{\delta}$. As expected, the first column in Table \ref{tb:exper-baselines} confirmed our theory. The first column in Table \ref{tb:exper-baselines} also suggests an approximately linear dependency between BAIR's stopping time and $K$. Although it is not fully supported by our theory (the leading term in the upper bound result is $O(K^{1.5})$ when $\alpha=1$), we believe this observation is informative and could be an interesting target for future work. 

\subsection{Experiments on Different Choices of $\Delta_1$}\label{app:Delta}
$\Delta_1$ is an environment variable that determines the difficulty of each problem instance; but different values of $\Delta_1$ do not impair the strength of BAIR. The comparison in Table \ref{tb:exper-baselines} (where $\Delta_1=0.5$) is illustrated in Figure \ref{fig:subfig:a} while the comparison under $\Delta_1=0.2$ is shown in Figure \ref{fig:subfig:b}. When facing harder problem instances with a smaller $\Delta_1$, BAIR and T$\&$S both maintained similar performance in terms of our evaluation metrics. Meanwhile, BAIR also consistently outperforms T$\&$S. UNI and EXP3, however, both suffered from a clear drop in the probability of success, as a smaller $\Delta_1$ means the user needs more comparisons before finding the best arm and thus incurs extra difficulty for UNI and EXP3 to distinguish those near-optimal arms.
\begin{figure*}[h]
\centering
\subfigure[$\Delta_1=0.5, \delta\in(0.1,0.05,0.02,0.01,0.005),K\in(2,5,20,100)$]{
\label{fig:subfig:a}
\includegraphics[width=1.0\textwidth]{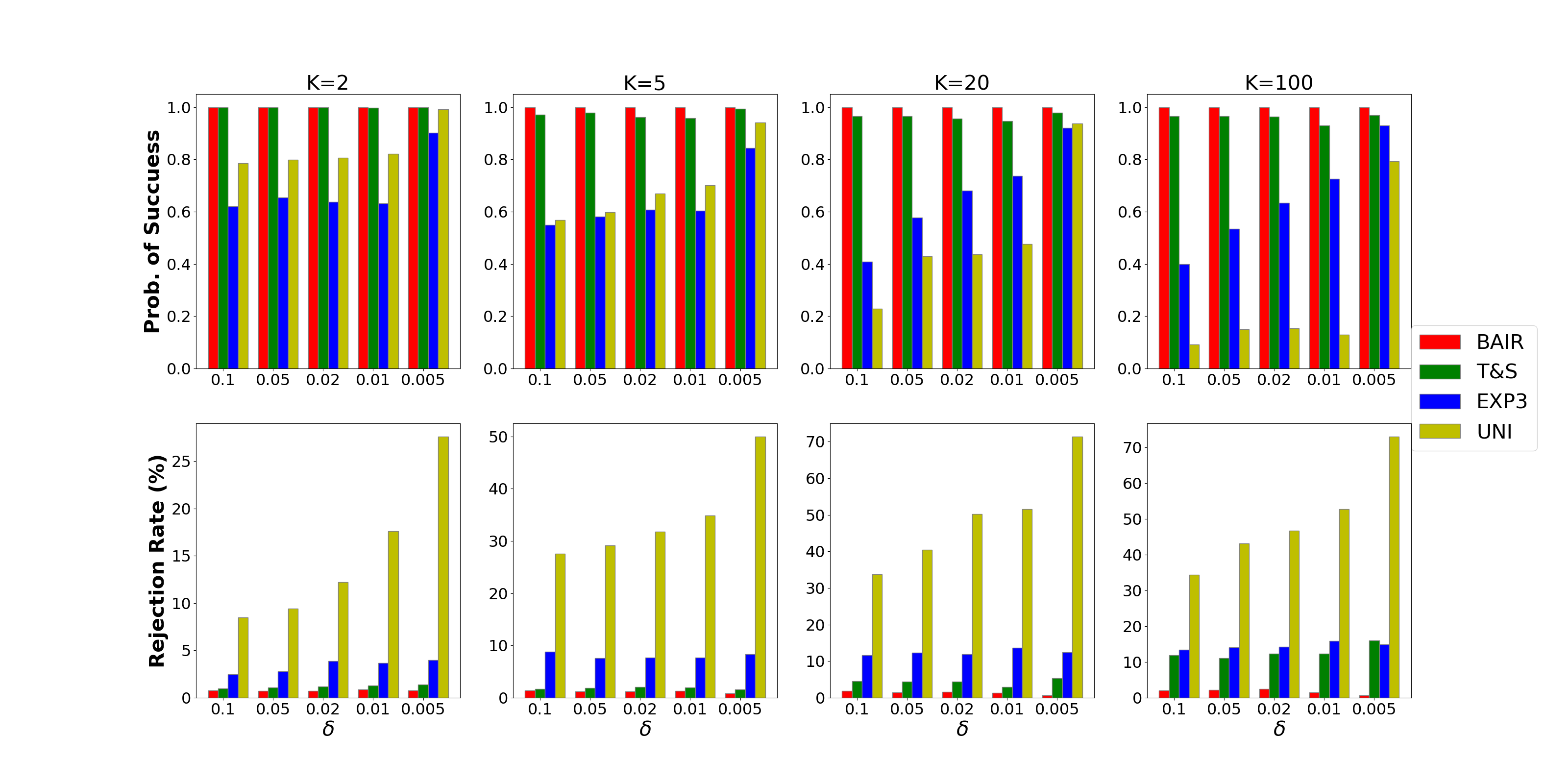}}
\subfigure[$\Delta_1=0.2, \delta\in(0.1,0.05,0.02,0.01,0.005),K\in(2,5,20,100)$]{
\label{fig:subfig:b}
\includegraphics[width=1.0\textwidth]{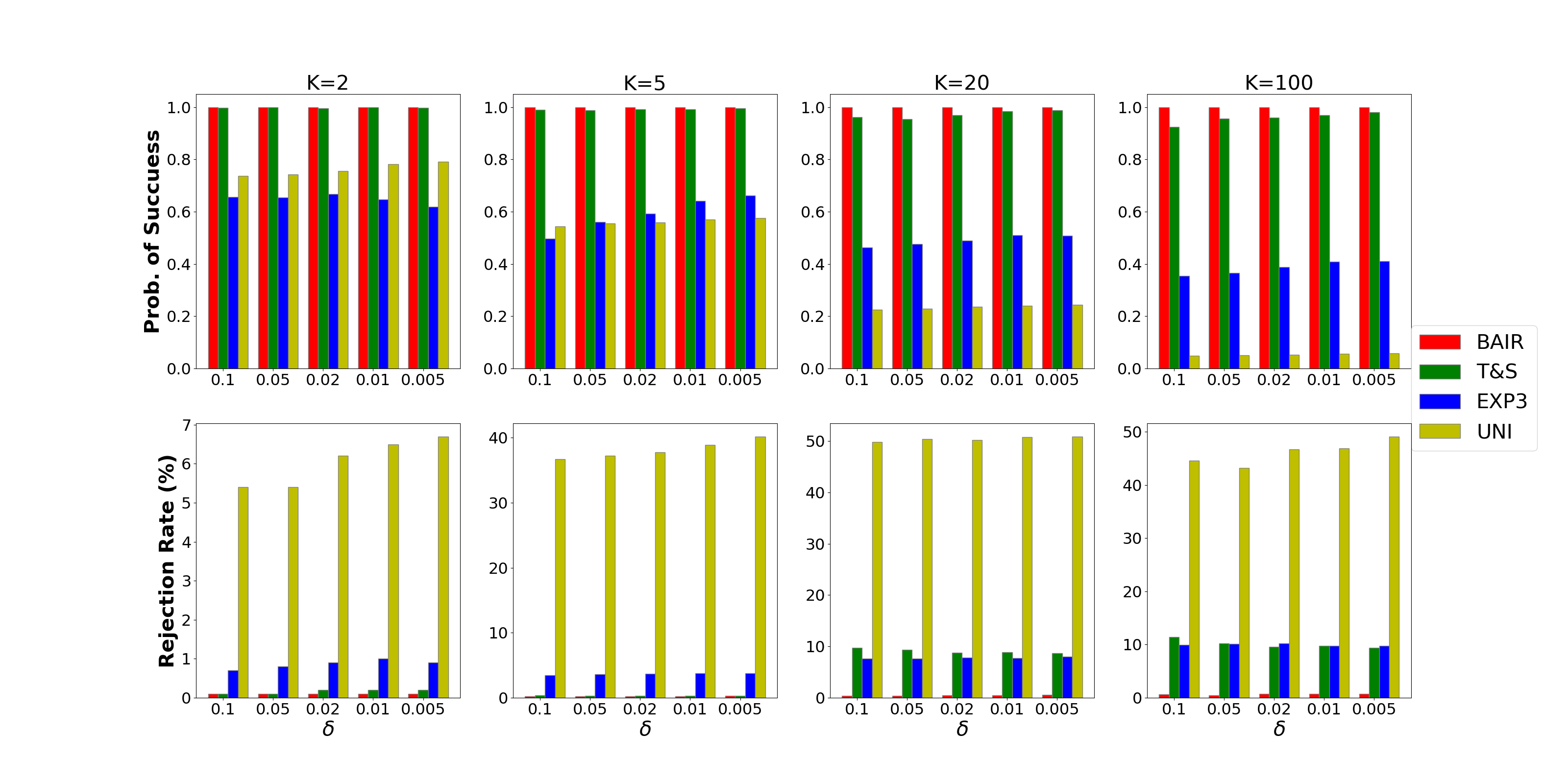}}

\caption{The comparison between BAIR/T\&S/EXP3/UNI for different $(\delta, K, \Delta_1)$. 
}
\label{fig:comparison} 
\end{figure*}

\subsection{Experiments on the Robustness of BAIR}\label{subsec:robust}
In practice, it might be too restrictive to assume that the user strictly follows our proposed confidence interval (CI) based behavior model. Thus, it is interesting and also crucial to test the robustness of BAIR under the situation where the user's behavior might deviate from the CI-based model. To this end, we extended the user model to a stochastic setting by incorporating ``decision randomness''. Specifically, we assume at each time step, with some constant probability $p$, the user makes a random decision (accept/reject the recommendation with an equal probability); otherwise, she would follow the CI-based behavior model. We demonstrate that a minor modification of BAIR still maintains competitive empirical performance in this new environment, against the three baselines. We use a natural variant of BAIR which adjusted its Phase-2 slightly: the system discards an arm after its $m$-th rejection rather than the very first rejection. We call this variant BAIR$(m)$; its adjustment is precisely to account for users' behavior noises. Note that with probability at least $1-(1-p)^m$, the first $m$-th rejection of the arm $i$ indicates $ucb_{i,t}<lcb_{j,t}$ for some $j$ and $t$ in Phase-2. By union bound, if we choose $m$ such that $1-(1-p)^m>1-\frac{\delta}{K}$ (i.e., $m>\log^{-1}(\frac{1}{1-p})\cdot\log\frac{K}{\delta}$), the Phase-2 in BAIR$(m)$ misidentifies the best arm with probability at most $\delta$, and BAIR$(m)$ finds the best arm with probability $1-2\delta$ in this new setting. Although the choice of $m$ incurs additional cost in Phase-2, it does not affect the order of our upper/lower bound result, and thus demonstrates the robustness of our solution. Table \ref{tb:exper-baselines_new} shows the comparison under the new noisy user model with $p=0.1$. The same comparison with the three baselines is reported in Table \ref{tb:exper-baselines_new}. Except for the new user model, the environment remains the same as the one used in Table \ref{tb:exper-baselines} in the main paper.

\begin{table}[t]
\centering
\begin{small}
\caption{Results under the extended user model with click noise $p=0.1$ and $m=2\log\frac{K}{\delta}$. Compared to the results in Table 2 (where $p=0$), BAIR performs slightly worse in terms of stopping time and probability of success, but still outperforms the other baselines. Its rejection rate becomes worse, because of the noisy feedback in Phase-1. However, the rejection rate is not as crucial as stopping time and probability of success for best arm identification.}
\label{tb:exper-baselines_new}
\begin{tabular}{ll|cc|cccc|cccc}
\toprule
 &   & \multicolumn{2}{c}{Stopping Time} & \multicolumn{4}{c}{Rejection Rate (\%)} & \multicolumn{4}{c}{Prob. of Success} \\
$\delta$ & $K$ & BAIR & T\&S & BAIR & UNI& EXP3 & T\&S & BAIR & UNI & EXP3 & T\&S \\
\midrule
\multirow{4}{*}{$0.1$} & 2    & 470  & 494 & 9.0 & 13.7  &  7.5 & 8.4 & 0.994  & 0.802  & 0.642 & 0.978 \\
& 5                           &  905 & 2234 & 13.9 & 33.3 & 12.6  & 7.3 & 0.973  & 0.606  & 0.587  & 0.967 \\
& 20                          & 2939 & 8832 & 19.3 & 41.2 & 17.6 & 10.2 & 0.966 & 0.292 & 0.563  & 0.959 \\
& 100                         & 13764 & 21162 & 19.8 & 39.0 & 18.0  & 20.2& 0.911 & 0.085  & 0.285  & 0.891\\
\hline

\multirow{4}{*}{$0.05$} & 2   & 506  & 516 & 10.4 & 15.1 & 7.7  & 8.2 & 1.000 & 0.832  & 0.617  & 0.997 \\
& 5                           & 959 & 2164 & 15.6 & 35.6 & 12.2 & 8.5 & 0.993 & 0.629  & 0.572 & 0.989  \\
& 20                          & 3062 & 10097 & 22.4 & 42.6 & 16.5 & 10.2 & 0.974 & 0.245 & 0.565  & 0.970 \\
& 100                         & 14393& 25321 & 26.0 & 40.4 & 18.5 & 18.3 & 0.965 & 0.064  & 0.645 & 0.930 \\

\bottomrule
\end{tabular}
\end{small}
\end{table}

\subsection{Experiments on the Comparisons with Shared Phase-1}

As we have discussed, the baseline algorithms fail due to the inadequate preparation to inform the user. One might wonder whether the baseline algorithms can be strengthened by some tailored preparation procedures, e.g., a straightforward plugin of Phase-1 algorithm. However, this idea does not work due to the following reasons. First, the output of Phase-1 (the number of rejections and acceptances on each arm) cannot be utilized by those baseline algorithms as they require the actual rewards, while the Phase-2 of BAIR is built on such simple information accumulated from Phase-1. Although it might be possible to redesign the baseline algorithms to allow them to leverage the output of Phase-1, this requires specific designs and should be considered a new algorithm, which is beyond the scope of this work. Second, if we simply run the same Phase-1 step for all algorithms without utilizing its output, there is no reason to believe such algorithms can outperform BAIR in Phase-2. This is because BAIR can confidently eliminate any arm after its very first rejection in Phase-2, while other algorithms simply do not have such confidence without the information provided by Phase-1. Consider EXP3 and T\&S: if two arms are both rejected less than once, but one has accumulated more acceptances than another, it is still insufficient to distinguish them with high confidence. However, only BAIR will utilize this fact to keep exploring these two arms. To summarize, Phase-1 and Phase-2 need to work together in BAIR to guarantee good performance; if other algorithms need to take advantage of Phase-1, we need to redesign them to incorporate the information obtained in Phase-1. Otherwise, a ‘prepared’ user after Phase-1 is still opaque to the baseline algorithms, as their response is only the relative preference rather than the actual reward. 

Our additional experiment result in Table \ref{tb:exper-extra2} also demonstrates our argument. With a shared phase-1 exploration, the UNI/EXP3 exhibits even worse success rate performance than BAIR, given the same stopping time in Phase-2. This is expected since many arms will have a very wide confidence interval at the beginning of Phase-2 and thus look equally good to the system for a long time if the output of Phase-1 is not utilized. This actually makes the user’s response even more confusing to the baseline algorithms. For T\&S, even though the stopping time has improved when it is small, the success rate suffers a significant loss due to the same reason.

\begin{table*}[t]
\centering
\begin{small}
\caption{Performance of baselines with shared Phase-1. $(\alpha=1,  N_1=\frac{2K}{\delta})$}
\label{tb:exper-extra2}
\begin{tabular}{ll|cc|cccc|cccc}
\toprule
 &   & \multicolumn{2}{c}{Phase-2 Stopping Time} & \multicolumn{4}{c}{Rejection Rate (\%)} & \multicolumn{4}{c}{Prob. of Success} \\
$\delta$ & $K$ & BAIR & T\&S & BAIR & UNI& EXP3 & T\&S & BAIR & UNI & EXP3 & T\&S \\
\midrule
\multirow{4}{*}{$0.02$} & 2 & 372 & 124 & 0.7 & 0.2 & 0.9 & 1.1 & 1.000 & 0.623 & 0.638  & 0.589 \\
                        &  5 & 479 & 326 & 1.2 & 1.2 & 1.7 & 1.8 & 1.000 & 0.460 & 0.428  & 0.492 \\
                        &  20  & 1693 & 2134 & 1.6 & 10.2 & 1.9 & 2.4 & 1.000 & 0.342& 0.280& 0.337 \\
                        &  100& 1628 & 4151 & 2.5 & 12.4 & 1.2 & 2.3 & 1.000 & 0.110 & 0.144 & 0.203\\
\hline
\multirow{4}{*}{$0.01$} & 2   & 237  & 189 & 0.8 & 0.5 & 0.7 & 1.7 & 1.000 & 0.612  & 0.620  & 0.570 \\
& 5                           & 140 & 413 & 1.3 & 1.2 & 0.7  & 2.0 & 1.000 & 0.430  & 0.304  & 0.438  \\
& 20                          & 217 & 3223& 1.4 & 13.5 & 1.6 & 2.2& 1.000 & 0.267  & 0.271  & 0.245 \\
& 100                         & 544 & 5754 & 1.5 & 14.4 & 1.9 & 2.4 & 1.000 & 0.091  & 0.125  & 0.181 \\
\bottomrule
\end{tabular}
\end{small}
\end{table*}

\subsection{Additional Experiments on Different Choices of $\alpha$}\label{app:additional_alpha}

Table \ref{tb:exper-baselines2}, \ref{tb:exper-baselines08} demonstrate the performance of BAIR against three baselines under the choices of $\alpha=2.0$ and $\alpha=0.8$. In the case $\alpha=2.0$, where the system serves an overly optimistic user, BAIR keeps outperforming the other three baselines and enjoys an even larger margin when $\frac{K}{\delta}$ becomes larger --- it is because a larger $\alpha$ reduces the total number of recommendations in Phase-1, which happens to be the dominant part in the stopping time when $\frac{K}{\delta}$ is large. On the other hand, a very explorative user with a larger $\alpha$ would generate feedback sequences that deviate further from the independent and stationary assumptions imposed in classical bandit solutions (e.g., our baselines) and lead to worse performance when those baseline methods are applied. 

When $\alpha=0.8$, BAIR's probability of success still outperforms the baselines', but its stopping time appears not as good as T$\&$S when $K$ becomes large. This is likely to be caused by our conservative choice of the upper bound of $N_1=O\Big(\Big(\frac{K}{\delta}\Big)^{\frac{1}{\alpha}}\Big)$ (to guarantee success rate), which may not be tight in $K$. Specifically, there is an $O(K^{\frac{1}{\alpha}})$ difference between the upper bound and the lower bound of $N_1$. As a result, when $\alpha$ is small, the choice of $N_1=\frac{1}{\rho_0}\Big(\frac{2K}{\delta}\Big)^{\frac{1}{\alpha}}$ might become too pessimistic. In this case, as a simple heuristic, one may adjust $N_1$ to balance the stopping time and the probability of success, e.g., by applying binary search within $(0, (2K/\delta)^{1/\alpha}/\rho_0)$ until a satisfactory probability of success is achieved on validation datasets. To demonstrate this tradeoff, we report the stopping time and the probability of success under the same setting as Table \ref{tb:exper-baselines08}, but with different choices of $N_1$ in Table \ref{tb:exper-N1}. As we can see, $N_1=K$ (i.e., only pull each arm once in Phase-1) is insufficient to guarantee $1-\delta$ probability of success, and the choice of $N_1=\frac{1}{\rho_0}\Big(\frac{2K}{\delta}\Big)^{\frac{1}{\alpha}}$ wastes too much time in Phase-1 especially for a large $K$ and small $\alpha$. However, there are choices in between that appear to be better tradeoffs as they increase the probability of success drastically at a moderate cost of stopping time. However, as we have discussed in Section \ref{sec:discussion}, the optimal choice of $N_1$ remains an open question.

\begin{table*}[t]
\centering
\begin{small}
\caption{Comparison between BAIR and three baselines on proposed metrics. $(\alpha=2.0)$}.
\label{tb:exper-baselines2}
\begin{tabular}{ll|cc|cccc|cccc}
\toprule
 &   & \multicolumn{2}{c}{Stopping Time} & \multicolumn{4}{c}{Rejection Rate (\%)} & \multicolumn{4}{c}{Prob. of Success} \\
$\delta$ & $K$ & BAIR & T\&S & BAIR & UNI& EXP3 & T\&S & BAIR & UNI & EXP3 & T\&S \\
\midrule

\multirow{4}{*}{$0.1$} & 2    &  933 & 1312 & 0.3 & 5.7  & 1.4 & 0.4 &  1.000 & 0.779  & 0.589  &  1.000\\
& 5                           &  1694 & 2934 & 0.5 &  26.0 & 4.9 & 0.4 & 1.000  & 0.455  &  0.502 & 1.000 \\
& 20                          & 4752  & 6737 & 0.8 & 27.6  & 7.7 & 1.2 & 1.000  & 0.265  &  0.402 & 1.000 \\
& 100                         &  20081 & 28430 & 0.8 & 32.0  & 7.8 & 5.2 & 1.000  &  0.025 &  0.333 & 0.988 \\
\hline
\multirow{4}{*}{$0.05$} & 2  & 956  & 1325 & 0.3 & 5.9  & 1.5 & 0.5 &  1.000  & 0.802  &  0.645 &  0.999\\
                        & 5  & 1698 & 3132 &  0.5 & 27.1 &5.2  &  0.6 &    1.000 & 0.459  & 0.551 & 0.989\\
                        & 20  & 4832 & 6799 & 0.8  & 28.0 & 8.8 &  1.3 &   1.000 &  0.278 & 0.459 & 0.962\\
                        & 100 & 22539 & 29730 & 0.9  & 29.9 & 9.0 &  4.8 & 1.000  & 0.030  & 0.354 & 0.930 \\
\hline
\multirow{4}{*}{$0.02$} & 2  & 965 & 1346 &  0.3 & 7.3 &  1.8&  0.8 & 1.000  & 0.835  & 0.761  & 1.000 \\
                        & 5  & 1709 & 3412 &  0.5 & 24.9 & 5.3 & 0.8 &  1.000  & 0.476  & 0.572  &  0.993\\
                        & 20  & 4850 & 6820 &  0.7 & 28.1 & 9.1 & 1.3 &  1.000 &  0.301 & 0.548  & 0.974 \\
                        & 100  & 23688 & 31241 & 0.9  & 30.8 & 9.2 &  4.7& 1.000  & 0.036  & 0.398 &0.954\\
\hline
\multirow{4}{*}{$0.01$} & 2   &  988 & 1379 & 0.3 & 9.8  &2.2  & 0.9 & 1.000  & 0.895  & 0.879  &  1.000\\
& 5                           & 1745  & 4171 & 0.5 & 26.0  & 5.6 &  1.0& 1.000  &  0.503 & 0.708  & 0.990\\
& 20                          &  4854 & 7012 & 0.8 &  28.1 & 9.1 & 4.3 & 1.000  & 0.332  &  0.684 &  0.952\\
& 100                         & 25091  & 34158 & 1.0 &  34.1 &  9.2& 4.0 &  1.000 & 0.047  &  0.454 &  0.931\\
\hline
\multirow{4}{*}{$0.005$} & 2   & 1023  & 1424 & 0.3 & 10.5  & 2.3 & 1.0 & 1.000  &  0.921 & 0.905  & 0.999\\
& 5                            &  1780 & 4634 & 0.5 &  26.4 & 5.6 & 1.1 & 1.000  &  0.534 &  0.855 & 0.992 \\
& 20                           &  4892 & 7086 & 0.8 & 28.3  & 8.9 &1.6  & 1.000  & 0.345  & 0.762  & 0.955 \\
& 100                         &  25537 & 39578 & 1.0 & 34.4  & 9.4 & 3.6 & 1.000  & 0.056  & 0.532  & 0.922 \\
\bottomrule
\end{tabular}
\end{small}
\end{table*}

\begin{table*}[t]
\centering
\begin{small}
\caption{Comparison between BAIR and three baselines on proposed metrics. $(\alpha=0.8)$}
\label{tb:exper-baselines08}
\begin{tabular}{ll|cc|cccc|cccc}
\toprule
 &   & \multicolumn{2}{c}{Stopping Time} & \multicolumn{4}{c}{Rejection Rate (\%)} & \multicolumn{4}{c}{Prob. of Success} \\
$\delta$ & $K$ & BAIR & T\&S & BAIR & UNI& EXP3 & T\&S & BAIR & UNI & EXP3 & T\&S \\
\midrule

\multirow{3}{*}{$0.1$} & 2    & 313  & 407 &  1.0 &  10.2  &  3.2 & 1.2  &  0.999 &  0.804 &  0.654  & 0.990  \\
& 5                           &  571 & 835 & 2.0  &  29.2  &  8.7 &  2.3 &  1.000 & 0.622  &  0.528  &  0.985 \\
& 20                          &  2070 & 2272 &  2.7 &  42.8  & 14.6  & 7.3  & 1.000  &  0.310 &  0.516  & 0.906  \\
\hline
\multirow{3}{*}{$0.05$} & 2   & 325  & 412 & 1.0  &  10.9  &  3.5 & 1.3  & 1.000  &  0.825 &  0.685  &  1.000 \\
& 5                           & 742  & 877 &  1.7 &  35.5  & 9.7  &  2.3 & 1.000  &  0.660 & 0.616   &  0.970 \\
& 20                          & 4310  & 7056 & 1.3  &  63.2  & 15.1  & 2.9  & 1.000  &  0.732 &  0.827  & 0.954  \\
\hline
\multirow{3}{*}{$0.02$} & 2   & 432  & 414 &  0.9 &  16.3  &  4.0 & 1.4  &  1.000 &  0.898 &  0.708  &  0.997 \\
& 5                           & 1883  & 1018 &  0.7 & 57.2   &  8.9 & 2.0  & 0.999  & 0.990  &  0.934  & 0.993  \\
& 20                          & 12785  & 8971 & 0.5  &  82.9  &  10.7 &  3.5 &  1.000 & 1.000  &  0.989  & 0.960  \\
\hline
\multirow{3}{*}{$0.01$} & 2   &  789 & 429 & 0.5  & 28.9  & 4.4  &  1.6 & 1.000  & 0.994  &  0.903  &  0.998 \\
& 5                           & 4321  & 1254 & 0.3  & 71.1   & 6.2  &  2.5 & 1.000  &  0.998 &  0.994  &  0.984 \\
& 20                          & 29895  & 10312 & 0.2  &   90.2 & 7.1  &  3.7 & 1.000  & 1.000  &  0.997  & 0.985  \\
\hline
\multirow{3}{*}{$0.005$} & 2  & 1821   & 436   & 0.2  &  40.4  & 3.4  &  1.8 &  1.000 & 0.999  &  0.996  &  0.998 \\
& 5                           & 10228  & 1460  & 0.1  &  76.9  &  4.2 & 2.0  & 1.000  & 1.000  &  1.000  &  0.994 \\
& 20                          & 71225  & 17934 & 0.1  &  92.9  & 4.7  &  2.6 & 1.000  &  1.000 & 1.000   &  0.992 \\
\bottomrule
\end{tabular}
\end{small}
\end{table*}

\begin{table*}[ht]
\centering
\caption{The stopping time and probability of success of BAIR under different choices of $N_1$. }
\label{tb:exper-N1}
\begin{small}
\begin{sc}
\begin{tabular}{ll|cccc|cccc}
\toprule
 &   & \multicolumn{4}{c}{Stopping Time}  & \multicolumn{4}{c}{Prob. of Success} \\
$\delta$ & $K$ & $N_1=(\frac{2K}{ \delta})^{\frac{1}{\alpha}}$ & $(\frac{\sqrt{K}}{ \delta})^{\frac{1}{\alpha}}$ & $(\frac{\log{K}}{ \delta})^{\frac{1}{\alpha}}$  & $K$ & $N_1=(\frac{2K}{ \delta})^{\frac{1}{\alpha}}$ & $(\frac{\sqrt{K}}{ \delta})^{\frac{1}{\alpha}}$  & $(\frac{\log{K}}{ \delta})^{\frac{1}{\alpha}}$ &  $K$ \\
\midrule
\multirow{2}{*}{$0.1$} & 20   & 2070 & 1673 & 1622  & 1478   & 1.000 & 1.000 & 1.000& 0.981\\
& 100                         & 13864 & 8860&8603    & 8207   & 1.000 & 1.000 & 0.999& 0.979\\
\multirow{2}{*}{$0.01$} & 20  & 29895 &  2251&1651  & 1572  & 1.000 & 1.000 & 0.999& 0.987  \\
& 100                         & 238579 & 8900&8660 & 8395 & 1.000 & 0.999 & 0.999&  0.980\\
\bottomrule
\end{tabular}
\end{sc}
\end{small}
\end{table*}
\section{Discussions and Future Work}\label{sec:discussion}

To bring user modeling to a more realistic setting in modern recommender systems, we proposed a new learning problem of best arm identification from explorative users' revealed preferences. We relax the strong assumptions that users are omniscient by modeling users' learning behavior, and study the learning problem on the system side to infer user's true preferences given only the revealed user feedback. We proved efficient system learning is still possible under this challenging setting by developing a best arm identification algorithm with complete analysis, and also disclosed the intrinsic hardness introduced by the new problem setup. Our result illustrates the inevitable cost a recommender system has to pay when it cannot directly learn from a user's realized utilities. As concluding remarks, we point out some interesting open problems in this direction:
    
\noindent\textbf{The optimal choice of $N_1$.} 
Although our lower bound result in Theorem \ref{thm:lowerbound4delta} is tight in $\delta$, it does not match the upper bound in Theorem \ref{thm:upperbound2} in terms of $K$. The mismatch comes from the choice of $N_1=(2K/\delta)^{1/\alpha}/\rho_0$, which might be overly pessimistic as Theorem \ref{thm:lowerbound4delta} only indicates a necessary condition of $N_1>\delta^{-1/\alpha} / \rho_0$. To bridge this gap, a tighter upper bound is needed to improve the choice of $N_1$. We believe this is promising because the experiment results in Table \ref{tb:exper-baselines} demonstrate that the choice of $N_1=(2K/\delta)^{1/\alpha}/\rho_0$ almost guarantees a success probability $1.0$ even when $\delta$ takes a large value, e.g., $0.1$. This implies the stopping time of BAIR could be improved by setting a smaller $N_1$. In practice, we can fine-tune $N_1$ to pin down the optimal choice. For example, one can simply apply binary search within $(0, (2K/\delta)^{1/\alpha}/\rho_0)$ with $N_1=O(\delta^{-1/\alpha})$ as a starting point.

\noindent\textbf{Beyond a single user.}
We note that our problem formulation and solution for the system and a single user also shed light on learning from a \emph{population} of users. For example, users sometimes learn or calibrate their utility from third-party services that evaluate the quality of items by aggregating users' feedback across different platforms.
As a result, users equipped with these services are inclined to exhibit an exploratory pattern and make decisions based on the comparison of confidence intervals.
We believe that our problem setting also provides a prototype to study the optimal strategy for the system under this new emerging situation.

\section{Acknowledgement}
This work is supported in part by the US National Science Foundation under grants IIS-2007492, IIS-1553568 and IIS-1838615. Haifeng Xu is supported by a Google Faculty Research Award. 

\vskip 0.2in

\bibliography{main}
\appendix
\newpage

\section{Omitted Proofs from Section \ref{subsec:ubphase2}} 
\subsection{Proof of Lemma \ref{lm:phase2bound1}}
\begin{proof}
When Algorithm \ref{alg:phase2_alg} terminates, exactly $K-1$ eliminations have happened. Suppose arm-$K$ survives and arm-$i$ ($1\leq i \leq K-1$) is eliminated at time step $t^e_i>N_1$. Denote $j_i$ as the arm with the highest lower confidence bound at $t=t^e_i$. Because $i$ is rejected at $t=t_i^e$, we have $lcb_{j_i,t_i^e}>ucb_{i,t_i^e}$. 

Define event $B_{i,t}: |\mu_{i}-\hat{\mu}_{i,t}| \leq \sqrt{\frac{\Gamma(t;\rho,\alpha)}{n_i^{t}}}, \forall i\in[K]$. Note that for any $1\leq i\leq K-1$, $B_{i,t_i^e} \cap B_{j_i,t_i^e}$ means $\mu_i$ and $\mu_{j_i}$ are both in their confidence intervals and thus the best arm cannot be wrongly eliminated at $t=t_i^e$. Therefore, from Hoeffding’s inequality and a union bound we obtain
\begin{align}\label{eq:badprob}  \notag
    & \mathbb{P}[K \text{~is the best arm}] \\ \notag \geq & \mathbb{P}(\cap_{i=1}^{K-1} [B_{i,t_i^e}\cap B_{j_i,t_i^e}]) \\ \notag
    =& 1- \mathbb{P}(\cup_{i=1}^{K-1} [B^c_{i,t_i^e}\cup B^c_{j_i,t_i^e}])\\  \notag
    \geq & 1-2\sum_{i=1}^{K-1} \exp{\left(-\frac{1}{2}n_i^{t^e_i}\cdot \frac{\Gamma(t^e_i; \rho,\alpha)}{n_i^{t^e_i}}\right)} \\  
     > & 1-\frac{2(K-1)}{(\rho_0 N_1)^{\alpha}} . 
\end{align}

Let the left-hand side of Eq \eqref{eq:badprob} be $1-\delta$, we obtain $N_1 \geq \frac{[2(K-1)]^{\frac{1}{\alpha}}}{\rho_0\delta^\frac{1}{\alpha}}.$
\end{proof}
We note that the lower bound of $N_1$ given in Lemma \ref{lm:phase2bound1} is loose in $K$ as we apply the union bound over $K$ arms. The choice of $N_1$ can be more flexible in practice, which is empirically verified in our experiments.

\subsection{Proof of Lemma \ref{lm:phase2bound2}}
\begin{proof}
From Lemma \ref{lm:iteratedlog}, we have for each arm $i$, 
$$\mathbb{P}\left(\forall n_i^t\in \mathbb{N}^+: |\mu_i-\hat{\mu}_{i,t}| \leq \sqrt{\frac{2\log \frac{Kn_i^t(n_i^t+1)}{\delta}}{n_i^t}} \right) > 1-\frac{\delta}{K}.$$
From a union bound, we have with probability $1-\delta$,
\begin{equation}\label{eq:3c}
|\mu_i-\hat{\mu}_{i,t}| \leq \sqrt{\frac{2\log \frac{Kn_i^t(n_i^t+1)}{\delta}}{n_i^t}}, \forall t\in \mathbb{N}^+, i\in [K]    
\end{equation}

Next, we upper bound the number of recommendations in Phase-2 to achieve $K-1$ rejections. Suppose $i=1$ is the best arm, a sufficient condition for eliminating any arm $i>1$ at time $t$ is
\begin{equation}\label{eq:4c}
\begin{aligned}
& |\mu_{i}-\hat{\mu}_{i,t}|<\frac{\Delta_i}{4}, \sqrt{\frac{\Gamma(t;\rho,\alpha)}{n_i^t}}<\frac{\Delta_i}{4}, |\mu_{1}-\hat{\mu}_{1,t}|<\frac{\Delta_i}{4}, \sqrt{\frac{\Gamma(t;\rho,\alpha)}{n_1^t}}<\frac{\Delta_i}{4},    
\end{aligned}
\end{equation}
since the inequalities in Eq \eqref{eq:4c} imply $lcb_{1,t}-ucb_{i,t} > (\hat{\mu}_{1,t}-\frac{\Delta_i}{4})-(\hat{\mu}_{i,t}+\frac{\Delta_i}{4}) > (\mu_{1,t}-\frac{\Delta_i}{2})-(\mu_{i,t}+\frac{\Delta_i}{2}) = 0.$ Given Eq \eqref{eq:3c}, a sufficient condition for Eq \eqref{eq:4c} is 

\begin{equation}\label{eq:5c}
\begin{aligned}
& \sqrt{\frac{2\log \frac{Kn_i^t(n_i^t+1)}{\delta}}{n_i^t}}<\frac{\Delta_i}{4}, \sqrt{\frac{2\alpha\log \rho_1 n(t)}{n_i^t}}<\frac{\Delta_i}{4}, \\ & \sqrt{\frac{2\log \frac{Kn_1^t(n_1^t+1)}{\delta}}{n_1^t}}<\frac{\Delta_i}{4}, \sqrt{\frac{2\alpha\log \rho_1 n(t)}{n_1^t}}<\frac{\Delta_i}{4}.
\end{aligned}
\end{equation}

Because the system always recommends the arm with the least number of pulls in Phase-2, we have $n_1^t \geq n_i^t$ when $i$ is eliminated. Observe that functions $f_1(n)=\sqrt{\frac{2\log \frac{Kn(n+1)}{\delta}}{n}}, f_2(n)=\sqrt{\frac{2\alpha\log \rho_1 n}{n}}$ are both decreasing when $n\geq 2$, the first two equations in Eq \eqref{eq:5c} imply the last two. Hence, Eq \eqref{eq:5c} can be satisfied if $n_i^t$ satisfies 
\begin{equation}\label{eq:6c}
\sqrt{\frac{2\log \frac{Kn_i^t(n_i^t+1)}{\delta}}{n_i^t}}<\frac{\Delta_i}{4}, \sqrt{\frac{2\alpha\log \rho_1 n}{n_i^t}}<\frac{\Delta_i}{4}. 
\end{equation}

Next we upper bound $n(t)$ in the term $2\alpha\log \rho_1 n(t)$. Suppose Phase-2 starts at $t=t_1$, and $j$ is the arm with the largest number of pulls at the beginning of Phase-2. Considering the following cases:
\begin{enumerate}
    \item if $n_j^t=n_j^{t_1}$, we have $n_s^t\geq n_j^t$ for all $s\neq j$. Therefore, $n(t)\leq K n_j^t = K n_j^{t_1} <KN_1\leq  \frac{2K^{1+1/\alpha}}{\rho_0 \delta^{1/\alpha}}.$
    \item if $n_j^t > n_j^{t_1}$, then there exists $t_2$ in Phase-2 such that when $t=t_2$, each arm must have been pulled exactly $n_j^{t_1}$ times. After $t=t_2$, the system starts to recommend all the arms in a round-robin manner. In this situation, the number of pulls for any pair of arms differs at most 1. Therefore, $n(t) \leq K \max_{s\in[K]}n_s^t \leq K (n_i^t+1).$
\end{enumerate}

Therefore, we conclude $\Gamma(t;\rho,\alpha) \leq 2\alpha\log \rho_1 n(t) \leq \max\{2\alpha\log\frac{ 2\rho_1K^{1+1/\alpha}}{\rho_0 \delta^{1/\alpha}}, 2\log \rho_1K n_i^t\}$. Substituting it into the second inequality of Eq \eqref{eq:6c}, we obtain 
\begin{equation}\label{eq:7c}
\begin{aligned}
&
\sqrt{\frac{2\log \frac{Kn_i^t(n_i^t+1)}{\delta}}{n_i^t}}<\frac{\Delta_i}{4},
\sqrt{\frac{2\alpha\log \frac{2\rho_1K^{1+1/\alpha}}{\rho_0\delta^{1/\alpha}}}{n_i^t}}<\frac{\Delta_i}{4},  \sqrt{\frac{2\alpha\log \rho_1K n_i^t}{n_i^t}}<\frac{\Delta_i}{4}. 
\end{aligned}
\end{equation}

Solving $n_i^t$ in Eq \eqref{eq:7c}, we obtain a sufficient condition for Eq \eqref{eq:7c} to hold:

\begin{equation}\label{eq:8c}
\begin{aligned}
& n_i^t > \max  \left\{\frac{32\alpha}{\Delta_i^2}\log \frac{2\rho_1K^{1+1/\alpha}}{\rho_0\delta^{1/\alpha}}, \frac{64\alpha}{\Delta_i^2}\log \frac{64\rho_1K}{\Delta_i^2}, \frac{128}{\Delta_i^2}\log \frac{64\sqrt{2K}}{\Delta_i^2\sqrt{\delta}}\right\},
\end{aligned}
\end{equation}

which yields $n_i^t \sim O(\frac{\alpha}{\Delta_i^2}\log \frac{\rho_1K}{\rho_0\delta\Delta_i}).$ Since Phase-2 incurs at most $K-1$ rejections, we conclude that Algorithm \ref{alg:phase2_alg} terminates within $O(K+\sum_{i=1}^K \frac{\alpha}{\Delta_i^2}\log \frac{\rho_1K}{\rho_0\delta\Delta_i})$ steps with probability $1-\delta$.
\end{proof}

\subsection{Proof of Lemma \ref{lm:Er}}
\begin{proof}
To simplify our notations, we omit the parameters $\rho$ and $\alpha$ in $\Gamma(t;\rho,\alpha)$ in Lemma \ref{lm:Er} and its proof. Suppose in round $[t_s^{(r)}, t_e^{(r)}]$, arm $1,2,\cdots,K$ are rejected successively at $t_s^{(r)}\leq t_1 < t_2 < \cdots < t_K=t_e^{(r)}$. Let $i$ be the arm with the highest empirical mean at $t=t_s^{(r)}$, i.e., $f(t_s^{(r)})=\hat{\mu}_{i,t_s^{(r)}}$. We will show that when any arm $j$ gets rejected, its empirical mean will be smaller than $f(t_s^{(r)})-2\sqrt{\frac{\underline{\Gamma}^{(r)}}{n(t_e^{(r)})}}$. Since any arm will not be pulled in the same round after being rejected, we conclude $f(t_e^{(r)}) \leq f(t_s^{(r)}) - 2\sqrt{\frac{\underline{\Gamma}^{(r)}}{n(t_e^{(r)})}}$. To prove our claim, we consider the following three cases:

\begin{enumerate}
    \item For any $j<i$, when arm $j$ is rejected at $t=t_j$, arm $i$'s empirical mean has not changed since $t=t_s^{(r)}$, i.e., $\hat{\mu}_{i,t_j}=\hat{\mu}_{i,t_s^{(r)}}$. Assume $j$ is rejected by $j'$, i.e., $ucb_{j,t_j}<lcb_{j',t_j}$. Then we have $\hat{\mu}_{j,t_j} < \hat{\mu}_{j',t_j}-\sqrt{\frac{\Gamma^(t_j)}{n_j^{t_j}}}-\sqrt{\frac{\Gamma^(t_j)}{n_{j'}^{t_j}}} \leq \hat{\mu}_{i,t_j}-2\sqrt{\frac{\underline{\Gamma}^{(r)}}{n(t_e^{(r)})}} \leq f(t_s^{(r)})-2\sqrt{\frac{\underline{\Gamma}^{(r)}}{n(t_e^{(r)})}}$ for all $j<i$.
    \item  When arm $i$ is rejected at $t=t_i$, let $j'\neq i$ such that $ucb_{i,t_i}<lcb_{j',t_i}$, i.e., $\hat{\mu}_{i,t_i}+\sqrt{\frac{\Gamma^(t_i)}{n_i^{t_i}}} < \hat{\mu}_{j',t_i}-\sqrt{\frac{\Gamma^(t_i)}{n_{j'}^{t_i}}}$. Then we obtain $\hat{\mu}_{i,t_i} < \hat{\mu}_{j',t_i}-\sqrt{\frac{\Gamma^(t_i)}{n_i^{t_i}}}-\sqrt{\frac{\Gamma^(t_i)}{n_{j'}^{t_i}}} \leq \hat{\mu}_{j',t_i}-2\sqrt{\frac{\underline{\Gamma}^{(r)}}{n(t_e^{(r)})}} \leq f(t_s^{(r)})-2\sqrt{\frac{\underline{\Gamma}^{(r)}}{n(t_e^{(r)})}}$.
    \item For any $j > i$, when arm $j$ is rejected at $t=t_j$, assume $j$ is rejected by $j'$, i.e., $ucb_{j,t_j}<lcb_{j',t_j}$. Note that if $j'<j$, we have $\hat{\mu}_{j',t_j}<\hat{\mu}_{i,t_s^{(r)}}$, otherwise $j'$ cannot be rejected before $j$; if $j'>j$, we also have $\hat{\mu}_{j',t_j}<\hat{\mu}_{i,t_s^{(r)}}$ because $\hat{\mu}_{j',t_j}=\hat{\mu}_{j',t_s^{(r)}}$. Therefore, $\hat{\mu}_{j,t_j} < \hat{\mu}_{j',t_j}-\sqrt{\frac{\Gamma^(t_j)}{n_j^{t_j}}}-\sqrt{\frac{\Gamma^(t_j)}{n_{j'}^{t_j}}} \leq \hat{\mu}_{i,t_s^{(r)}}-2\sqrt{\frac{\underline{\Gamma}^{(r)}}{n(t_e^{(r)})}} = f(t_s^{(r)})-2\sqrt{\frac{\underline{\Gamma}^{(r)}}{n(t_e^{(r)})}}$ for all $j>i$. 
    
    As a result, we conclude that $\hat{\mu}_{j,t_j} <f(t_s^{(r)})-2\sqrt{\frac{\underline{\Gamma}^{(r)}}{n(t_e^{(r)})}}$ for any $j \in [K]$, which means $f(t_e^{(r)}) \leq f(t_s^{(r)}) - 2\sqrt{\frac{\underline{\Gamma}^{(r)}}{n(t_e^{(r)})}}$.
\end{enumerate}

\end{proof}
\subsection{Proof of Lemma \ref{lm:algo4}}
\begin{proof}
First of all, observe that given any $N_1$, Algorithm \ref{alg:phase1_alg} must terminate because at least one arm will be accepted in each round. Therefore, for any $t>0$, the system can collect at least $\lfloor \frac{t}{K+1} \rfloor$ acceptances in the first $t$ steps. Suppose the initialization stage ends at $t=t_0$, and without loss of generality we start to index the rounds at $t_s^{(1)}=\max\{t_0, \frac{e^{1/2\alpha}(K+1)}{\rho_0}\}$. Then before $t=t_s^{(1)}$, the system will be rejected at most $\max\{K,\frac{e^{1/2\alpha}K}{\rho_0}\}$ times and be accepted at least $\frac{e^{1/2\alpha}}{\rho_0}$ times. Hence, we have $\Gamma(t)\geq 2\alpha\log [\rho_0 n(t)] \geq 2\alpha\log [\rho_0 \frac{e^{1/2\alpha}}{\rho_0}]=1, \forall t>t_s^{(1)}$.

Suppose Algorithm \ref{alg:phase1_alg} terminates in the $(M+1)$-th round denoted by $\{[t_s^{(r)}, t_e^{(r)}]\}_{r=1}^M$, where $t_s^{(1)}<t_e^{(1)}=t_s^{(2)}<t_e^{(2)}=t_s^{(3)}<\cdots<t_e^{(M)}\leq N_1.$ Next, we derive the upper bound of $M$ and the upper bound for the total number of rejections is thus given by $\max\{K,\frac{e^{1/2\alpha}K}{\rho_0}\}+KM \sim O(KM)$. From Lemma \ref{lm:Er}, we have 
\begin{align}
\notag
    f(t_e^{(M)}) & \leq f(t_s^{(M)}) - 2\sqrt{\frac{\Gamma(t_s^{(M)})}{n(t_e^{(M)})}} \\ \notag
    & \leq f(t_s^{(M)}) - 2\sqrt{\frac{1}{N_1}}\\  \notag
    & \leq f(t_s^{(M-1)}) - 2\cdot 2\sqrt{\frac{1}{N_1}}\\ \notag
    & \leq \cdots \\\label{eq:2mn}
    & \leq f(t_s^{(1)}) - 2M\sqrt{\frac{1}{N_1}}.
\end{align}

Suppose $f(t_s^{(1)}) = \hat{\mu}_{i,t_s^{(1)}}, f(t_e^{(M)})=\hat{\mu}_{j,t_e^{(M)}}\geq \hat{\mu}_{i,t_e^{(M)}}$. From Eq \eqref{eq:2mn} we obtain 

\begin{equation}\label{eq:2mn2}
     \hat{\mu}_{i,t_s^{(1)}} - \hat{\mu}_{i,t_e^{(M)}} \geq \frac{2M}{\sqrt{N_1}}.
\end{equation}

From Lemma \ref{lm:iteratedlog}, we know that with probability $1-\frac{\delta}{K}$, $$\mathbb{P}\Big(\forall t\in \mathbb{N}^+: |\hat{\mu}_{i,t}-\mu_i| \leq \sqrt{2\log \frac{2K}{\delta}} \Big) > 1-\frac{\delta}{K}.$$ Therefore, with probability $1-\delta$, 

\begin{equation}\label{eq:2mn3}
     \hat{\mu}_{i,t_s^{(1)}} - \hat{\mu}_{i,t_e^{(M)}} \leq 2\cdot\sqrt{2\log \frac{2K}{\delta}}.
\end{equation}

Eq \eqref{eq:2mn2} and \eqref{eq:2mn3} provide the following 
\begin{equation}\label{eq:2mn4}
\frac{2M}{\sqrt{N_1}}\leq 2\cdot\sqrt{2\log \frac{2K}{\delta}}.
\end{equation}

Rearranging Eq \eqref{eq:2mn4} yields
$$M \leq \sqrt{2N_1\log\frac{2K}{\delta}},$$ 
which completes the proof.
\end{proof}

\section{Proof of Theorem \ref{thm:upperbound2}}
\begin{proof}
Let $N_1=\frac{1}{\rho_0}\cdot\Big(\frac{2K}{\delta}\Big)^{\frac{1}{\alpha}}$ in Lemma \ref{lm:algo4}, we can upper bound the number of rejections in Phase-1 by $O\left(K^{1+\frac{1}{2\alpha}}\delta^{-\frac{1}{2\alpha}}\sqrt{\log \frac{K}{\delta}}\right)$ with probability $1-\delta$. From Lemma \ref{lm:phase2bound2}, we can upper bound the number of recommendations in Phase-2 by $O(K+\sum_{i=1}^K \frac{\alpha}{\Delta_i^2}\log \frac{\rho_1K}{\rho_0\delta\Delta_i})\sim O(\frac{\alpha K}{\Delta_1^2}\log \frac{K}{\delta\Delta_1})$ with probability $1-\delta$. Combining the results from Phase-1 and Phase-2, we complete the proof of Theorem \ref{thm:upperbound2}.
\end{proof}
\section{Proof of Theorem \ref{thm:lowerbound4delta}}

\begin{proof}\label{appendix:proof4lowerbound}
First we prove that the system needs at least $\frac{2}{\Delta_1^2}\log \frac{1}{4\delta}$ acceptances to identify the best arm with probability $1-\delta$. We need the general lower bound result for best arm identification with fixed confidence \cite{garivier2016optimal}, which is stated below.

Let $\mathcal{S} = \{\nu=(\mu_1,\cdots, \mu_K): \exists i_* \text{~s.t.~} \mu_{i_*} > \max\{\mu_i:i\neq i_*\}\}$ be the set of $K$-armed stochastic bandit instances where the reward distribution for each arm $i$ follows $N(\mu_i,1)$. Let Alt$(\nu)=\{\nu' \in \mathcal{S}:\max(\nu) \neq \max(\nu')\}$ and $\Sigma_K=\{\omega\in \mathbb{R}_+^K:\omega_1+\cdots+\omega_K=1\}$ be the set of probability distribution on $[K]$. From Theorem 1 in \cite{garivier2016optimal}, we know that for any policy $\pi$ that can identify the best arm with probability $1-\delta$ for any $\nu \in \mathcal{S}$, the expected stopping time $\tau_{\delta}$ of $\pi$ on any $\nu$ satisfies 
\begin{equation*}
    \mathbb{E}_{\nu}[\tau_{\delta}]\geq T^*(\nu)\log \frac{1}{4\delta},
\end{equation*}
where $T^*(\nu)^{-1}=\sup_{\omega \in \Sigma_K}\inf_{\nu' \in \text{Alt}(\nu)}\Big(\sum_{i=1}^K \omega_i D(\mu_i, \mu'_i)\Big)$, and $D(\mu_i, \mu'_i)$ denotes the Kullback-Leibler divergence of two distributions $N(\mu_i, 1), N(\mu'_i, 1)$. 

To derive an explicit lower bound, we assume $\mu_1>\mu_2\geq \cdots\geq \mu_K$ for $\nu$, and let $\nu'$ satisfy $\mu'_1=\mu_2-\epsilon $ and $\mu'_i=\mu_i, i \neq 1$. Then we have $\nu' \in$ Alt $(\nu)$ and 
\begin{align*}
    T^*(\nu)^{-1} & =\sup_{\omega \in \Sigma_K}\inf_{\nu' \in \text{Alt}(\nu)}\Big(\sum_{i=1}^K \omega_i D(\mu_i, \mu'_i)\Big) \\
    & \leq \sup_{\omega \in \Sigma_K}\inf_{\epsilon>0}\Big( \omega_1 D(\mu_1, \mu_2-\epsilon)\Big)\\
    & = \inf_{\epsilon>0} \frac{(\Delta_1+\epsilon)^2}{2} = \frac{\Delta_1^2}{2}. \\
\end{align*}
Therefore, we obtain
\begin{equation}\label{eq:lowertau}
\mathbb{E}_{\nu}[\tau_{\delta}]\geq \frac{2}{\Delta_1^2}\log \frac{1}{4\delta}.
\end{equation}

Suppose there exists an algorithm $\pi$ for the system such that for any $\delta \in (0,1)$ and $\nu \in \mathcal{S}$, $\pi$ can find the best arm with probability $1-\delta$ with less than $N_0 = \frac{2}{\Delta_1^2}\log \frac{1}{4\delta}$ accepted recommendations. Then, the user can also run algorithm $\pi$ on her side to identify the best arm with less than $\frac{2}{\Delta_1^2}\log \frac{1}{4\delta}$ pulls, which contradicts the lower bound given in Eq \eqref{eq:lowertau}. As a result, $N_0 \geq \frac{2}{\Delta_1^2}\log \frac{1}{4\delta}$ is required.

Next, we prove that if $N_0 < \frac{\delta^{-\frac{1}{\alpha}+c}}{\rho_0}$, there exists a problem instance such that the system must make mistake about the best arm with probability at least $1-\delta$. Consider the following problem instances,
\begin{align*}
    &\nu=(1+\epsilon-d, 1,-\frac{1}{\delta},\cdots,-\frac{1}{\delta}),\\
&\nu'=(1+\epsilon, 1,-\frac{1}{\delta},\cdots,-\frac{1}{\delta}),
\end{align*}
where $d=\sqrt{2\alpha\log\rho_1 N_0} (1+2\sqrt{\frac{1}{N_0-K+1}})+2\sqrt{\frac{\log 2(N_0-K+1)}{N_0-K+1}}, \epsilon=\sqrt{\frac{2\alpha\log\rho_0 N_0}{N_0-K+1}}$. Note that $N_0 \geq K$ since the first $K$ recommendations will always be accepted. Given this construction, we know the best arms for $\nu,\nu'$ are $2$ and $1$, respectively. Since only accepted recommendations affect the user's empirical reward estimation, we index the empirical means $\hat{\mu}_i$ with subscript $n$ (the number of acceptances so far) instead of $t$.  

We argue that for any policy $\pi$ that has made $N_0$ recommendations, $\pi$ will make mistake on either $\nu$ or $\nu'$ with probability greater than $\delta$. We will prove this claim by showing that when $\pi$ is applied to $\nu$ and $\nu'$, the induced empirical histories will overlap with probability larger than $2\delta$, and thus conclude $\pi$ is unable to distinguish $\nu$ and $\nu'$ within probability threshold $\delta$.

Specifically, we show that for sufficiently small $\delta$, with probability at least $2\delta$, any arm except arm 2 can only be pulled once during the first $N_0$ acceptances in both $\nu$ and $\nu'$. To see this, we first lower bound the probability of the following events for both $\nu$ and $\nu'$:
\begin{enumerate}
    \item Let event $A= \left\{\forall 1\leq n\leq N_0, i\in\{2,\cdots, K\}: |\hat{\mu}_{i,n} -\mu_{i}| < 2\sqrt{\frac{\log 2n_i}{n_i}} \right\}$.
  We claim $\mathbb{P}[A] > \frac{1}{2^{K-1}}$, because from Lemma \ref{lm:iteratedlog} we have
\begin{align*}
  \mathbb{P}[A]   &    = \prod_{i=2}^K  \mathbb{P}\Big[\forall 1\leq n\leq N_0: |\hat{\mu}_{i,n} -\mu_{i}| < 2\sqrt{\frac{\log 2n_i}{n_i}} \Big] \\
    & >  \prod_{i=2}^K  \mathbb{P}\Big[\forall n\in \mathbb{N}^+: |\hat{\mu}_{i,n} -\mu_{i}| < \sqrt{\frac{2\log 2n_i(n_i+1)}{n_i}} \Big] \\ 
    & > \prod_{i=2}^K \frac{1}{2} = \frac{1}{2^{K-1}}.
\end{align*}
    \item Define events $B=\{\hat{\mu}_{1,1} < 1+\epsilon-d\}$ and $B'=\{\hat{\mu}'_{1,1} < 1+\epsilon-d\}$. We claim $\mathbb{P}[B]>\mathbb{P}[B']>2^K \delta$, since we can directly lower bound $\mathbb{P}[B']$ as below:

\begin{small}
    \begin{align}
\label{eq:feller1968}
\notag
& \, \, \mathbb{P}[B'] =  \mathbb{P}[\hat{\mu}'_1-\mu'_1 < -d]\\ 
\geq &  \, \,   \frac{1}{\sqrt{2\pi}}\Big(\frac{1}{d}-\frac{1}{d^3}\Big)\cdot\exp\Big\{{-\frac{1}{2}d^2}\Big\}\\ \notag
= & \, \, \frac{1}{\sqrt{2\pi}}\Big(\frac{1}{d}-\frac{1}{d^3}\Big) \cdot\exp\Big\{{-\frac{1}{2}\Big(\sqrt{2\alpha\log\rho_1 N_0} (1+\frac{2}{\sqrt{N_0-K+1}})+2\sqrt{\frac{\log 2(N_0-K+1)}{N_0-K+1}}\Big)^2}\Big\}  \\  \label{eq:146}
> & \, \, \frac{1}{\sqrt{2\pi}}\Big(\frac{1}{d}-\frac{1}{d^3}\Big)\cdot\exp\Big\{{-\frac{1}{2}\Big(\sqrt{2\alpha\log\rho_1 N_0} +N_0^{-\frac{1}{4}}\Big)^2}\Big\}\\ \notag
= & \, \, \frac{1}{\sqrt{2\pi}}\Big(\frac{1}{d}-\frac{1}{d^3}\Big) \cdot\exp\Big\{{-\alpha\log\rho_1 N_0 }\Big\}\cdot\exp\Big\{{-N_0^{-\frac{1}{4}} \sqrt{2\alpha\log\rho_1 N_0}}\Big\}\cdot\exp\Big\{{-\frac{1}{2}N_0^{-\frac{1}{2}} }\Big\} \\ \label{eq:147}
\geq & \, \, \frac{1}{\sqrt{2\pi}}\Big(\frac{1}{2d}\Big) \cdot \frac{\rho_0^{\alpha}}{\rho_1^{\alpha}}\delta^{1-\alpha c} \cdot \exp\Big\{{-1}\Big\} \cdot \exp\Big\{{-\frac{1}{2} }\Big\} \\ \label{eq:148}
= & \, \, \frac{\rho_0^{\alpha}\exp{(-1.5)}}{8\sqrt{\pi}\rho_1^{\alpha}\cdot \delta^{\alpha c}\cdot\sqrt{\alpha \log \rho_1 N_0}} \cdot \delta> 2^K\delta. 
\end{align}
\end{small}
Note that Eq \eqref{eq:feller1968} holds because of the tail bounds theorem for Gaussian distribution \cite{feller1966introduction}, and Eq \eqref{eq:146},\eqref{eq:147},\eqref{eq:148} hold because we can choose sufficiently small $\delta$ such that $  2\sqrt{\frac{2\alpha\log\rho_1 N_0}{N_0-K+1}}+2\sqrt{\frac{\log 2(N_0-K+1)}{N_0-K+1}}<N_0^{-\frac{1}{4}}, N_0^{-\frac{1}{4}} \sqrt{2\alpha\log\rho N_0}<1, N_0^{-\frac{1}{2}}<1, $ and $\frac{\rho_0^{\alpha}\exp{(-1.5)}}{8\sqrt{\pi}\rho_1^{\alpha}\cdot \delta^{\alpha c}\cdot\sqrt{\alpha \log \rho_1 N_0}}>2^K$. Since $\mu'_1>\mu_1$, we further obtain $\mathbb{P}[B]>\mathbb{P}[B']>2^K \delta$.
\end{enumerate}

Since events $A, B, B'$ are independent, we have $\mathbb{P}[A \cap B]> 2\delta$ and $\mathbb{P}[A\cap B']> 2\delta$. Next we show that conditioned on the event $A \cap B$ or $ A\cap B'$, all the arms except arm 2 can only be pulled once. 

For any arm $i \in \{3,\cdots, K\}$ and any $t$ such that $1\leq n(t) \leq N_0$ (we use $n$ to denote $n(t)$ in the remaining proof for simplicity), we have 
  \begin{align} \notag
   & lcb_{2,t}-ucb_{i, t} \\ \notag =& \hat{\mu}_{2,n}-\sqrt{\frac{2\alpha\log \rho_t n}{n_2}} - \hat{\mu}_{i,n}-\sqrt{\frac{2\alpha\log \rho_t n}{n_i}} \\ \notag
     >& (1 - 2\sqrt{\frac{\log 2n_2}{n_2}}) -\sqrt{\frac{2\alpha\log \rho_t n}{n_2}} - (-\frac{1}{\delta} + 2\sqrt{\frac{\log 2n_i}{n_i}}) -\sqrt{\frac{2\alpha\log \rho_t n}{n_i}} \\  \notag
     >& \frac{1}{\delta}+1- 2\Big(2 \sqrt{\log 2N_0} +  \sqrt{2\alpha\log \rho_t N_0} \Big) \\ \label{eq:161}
     >& \frac{1}{\delta} -O\Big( \sqrt{\log \delta^{-1+\alpha c}} \Big) > 0, \delta \xrightarrow[]{} 0.
\end{align}
Eq \eqref{eq:161} means when $\delta$ is sufficiently small, arm $3,\cdots,K$ will always be rejected because of arm 2 during the first $N_0$ steps and thus can only be pulled once. 

For arm 1, suppose it has been pulled only once during the first $n-1$ steps, we show it will also be rejected at the $n$'th step. First, observe that the upper confidence bound for arm 1 at the $n$'th step satisfies
\begin{equation}\label{eq:ucbarm1}
    ucb_{1,n} \leq 1+\epsilon-d+\sqrt{2\alpha\log \rho_t n}.
\end{equation}
Since arm $1,3,\cdots,K$ are only pulled once, we have $n_2=n-K+1$, and the lower confidence bound for arm 2 at the $n$'th step thus satisfies
\begin{equation}\label{eq:ucbarm1}
    lcb_{2, n}\geq 1-2\sqrt{\frac{\log 2n_2}{n_2}}-\sqrt{\frac{2\alpha\log \rho_t n}{n_2}}.
\end{equation}
Therefore, 

\begin{align} \notag
   & lcb_{2,t}-ucb_{1,t} \\ \notag  = & 1-2\sqrt{\frac{\log 2(n-K+1)}{n-K+1}}-\sqrt{\frac{2\alpha\log \rho_t n}{n-K+1}} - (1+\epsilon-d+\sqrt{2\alpha\log \rho_t n}) \\ \notag
    =& d-\epsilon- \Big(\sqrt{2\alpha\log \rho_t n} +2\sqrt{\frac{\log 2(n-K+1)}{n-K+1}} + \sqrt{\frac{2\alpha\log \rho_t n}{n-K+1}} \Big) \\ \notag
     >& \Big\{\Big(\sqrt{2\alpha\log\rho_t N_0} (1+\sqrt{\frac{1}{N_0-K+1}})+2\sqrt{\frac{\log 2(N_0-K+1)}{N_0-K+1}}\Big) \\&- \Big(\sqrt{2\alpha\log\rho_t n} (1+\sqrt{\frac{1}{n-K+1}})+2\sqrt{\frac{\log 2(n-K+1)}{n-K+1}}\Big) \Big\}\nonumber \\ \label{eq:162}
     \triangleq & f(N_0;\rho_t,\alpha)-f(n;\rho_t,\alpha) > 0,
\end{align}
where Eq \eqref{eq:162} holds because we can verify that for any fixed $\rho_t \in [\rho_0, \rho_1], \alpha>0$, function $f(n;\rho_t,\alpha)$ is increasing and goes to infinity as $\delta \xrightarrow[]{} 0, n \xrightarrow[]{} +\infty$. Therefore, we can choose a sufficiently small $\delta$ such that $N_0$ is sufficient large and $f(N_0;\rho_t,\alpha)-f(n;\rho_t,\alpha)>0$ holds for all $n<N_0$.

Now we have shown that when $\pi$ is executed on either $\nu$ or $\nu'$, with probability at least $2\delta$, only arm 2 will be accepted more than once within the first $N_0$ accepted recommendations. Under this circumstance, the observations from $\nu$ and $\nu'$ are completely indistinguishable. Therefore, if the system has to output a guess for the best arm between $1$ and $K$, it is bound to make a mistake with probability larger than $\delta$, which completes the proof.
\end{proof}

\end{document}